\documentclass{article}

\PassOptionsToPackage{numbers, compress}{natbib}

\usepackage[usenames,dvipsnames]{xcolor}         %

\usepackage[final]{neurips_2023}

\usepackage[utf8]{inputenc} %
\usepackage[T1]{fontenc}    %
\usepackage{hyperref}       %

\hypersetup{
    colorlinks = true,
    citecolor = blue!60!black,
}

\usepackage{url}            %
\usepackage{booktabs}       %
\usepackage{amsfonts}       %
\usepackage{nicefrac}       %
\usepackage{microtype}      %

\usepackage{algorithm}
\usepackage{algorithmic}
\usepackage{soul}

\usepackage{url}            %
\usepackage{booktabs}       %
\usepackage{amsfonts}       %
\usepackage{nicefrac}       %

\usepackage{natbib}

\usepackage{float}
\usepackage{subcaption}
\usepackage{graphicx}
\usepackage[export]{adjustbox}
\usepackage[inline]{enumitem}

\usepackage{multirow}

\usepackage{amsmath,amssymb,amsthm}
\usepackage[mathscr]{eucal}
\usepackage{stmaryrd}
\usepackage{accents}
\usepackage{upgreek}
\usepackage[scaled=0.8]{FiraMono}

\usepackage{booktabs,colortbl,multirow}

\usepackage{pifont}

\usepackage{algorithm,algorithmic}
\usepackage{wrapfig}
\renewcommand{\algorithmicrequire}{\textbf{Input:}}
\renewcommand{\algorithmicensure}{\textbf{Output:}}

\graphicspath{{figs//}}

\allowdisplaybreaks

\def\1{\bm{1}}

\DeclareMathAlphabet{\mathsfit}{\encodingdefault}{\sfdefault}{m}{sl}
\SetMathAlphabet{\mathsfit}{bold}{\encodingdefault}{\sfdefault}{bx}{n}

\newcommand{\E}{\mathbb{E}}

\newcommand{\R}{\mathbb{R}}

\DeclareMathOperator*{\argmax}{argmax}
\DeclareMathOperator*{\argmin}{argmin}

\newcommand{\removefornow}[1]{}

\renewcommand{\Pr}{\mathbb{P}}

\newcommand{\XCal}{\mathscr{X}}
\newcommand{\YCal}{\mathscr{Y}}

\newcommand{\defEq}{\stackrel{.}{=}}

\newcommand{\Real}{\mathbb{R}}

\usepackage{comment}
\usepackage[capitalize,noabbrev]{cleveref}

\usepackage[toc,page,header]{appendix}
\usepackage{minitoc}

\crefformat{section}{\S#2#1#3}

\theoremstyle{plain}
\newtheorem{theorem}{Theorem}[section]
\newtheorem{proposition}[theorem]{Proposition}
\newtheorem{lemma}[theorem]{Lemma}
\newtheorem{corollary}[theorem]{Corollary}
\theoremstyle{definition}

\theoremstyle{remark}

\usepackage[textsize=tiny]{todonotes}

\newcommand{\ourtitle}{When Does Confidence-Based Cascade Deferral Suffice?}
\title{\ourtitle}

\author{
Wittawat Jitkrittum \qquad Neha Gupta \qquad Aditya Krishna Menon  \\[0.5em]
 \textbf{Harikrishna Narasimhan} \qquad \textbf{Ankit Singh Rawat} \qquad
\textbf{Sanjiv Kumar}  \\[0.8em]
Google Research, New York \\[0.8em]
\texttt{\{wittawat, nehagup, adityakmenon, hnarasimhan, ankitsrawat, sanjivk\}@google.com}
}

\begin{document}

\doparttoc %
\faketableofcontents %

\part{} %

\maketitle

\begin{abstract}
Cascades are a classical strategy to enable inference cost to vary \emph{adaptively} across samples,
wherein a sequence of classifiers are invoked in turn.
A \emph{deferral rule} determines whether to invoke the next classifier in the sequence, or to terminate prediction.
One simple deferral rule 
employs
the \emph{confidence} of the current classifier,
e.g.,
based on the maximum predicted softmax probability.
Despite being oblivious to the structure of the cascade 
---
e.g., not modelling the errors of downstream models
---
such confidence-based deferral often works remarkably well in practice.
In this paper, we seek to better understand 
the conditions under which confidence-based deferral may 
fail,
and when alternate deferral strategies can perform better.
We first present a theoretical characterisation of the optimal deferral rule, 
which precisely characterises settings under which confidence-based deferral may suffer.
We then 
study \emph{post-hoc} deferral mechanisms,
and demonstrate they
can significantly improve upon confidence-based deferral in settings where
(i)  downstream models are \emph{specialists} that only work well on a subset of inputs,
(ii) samples are subject to label noise,
and
(iii) there is distribution shift between the train and test set.

\end{abstract}

\section{Introduction}
Large neural models with several billions of parameters have shown considerable promise in challenging real-world problems, such as language modelling~\citep{Radford:2018,Radford:2019,Brown:2020,Raffel:2020} and image classification~\citep{Dosovitskiy:2021,Fedus:2021}.
While the quality gains of these models are impressive, they are typically accompanied with a sharp increase in inference time~\citep{Canziani:2016,Schwartz:2019}, thus limiting their applicability.
\emph{Cascades} offer one strategy to mitigate this~\citep{Viola:2001,Varshney:2022,saberian2014boosting,graf2004parallel},
by allowing for faster predictions on ``easy'' samples.
In a nutshell, cascades involve arranging multiple models in a sequence of increasing complexities.
For any test input, 
one iteratively 
applies the following recipe, starting with the first model in the sequence:
execute the current model, and employ 
a \emph{deferral rule} to determine whether to 
invoke the next model,
or
terminate with the current model's prediction.
One may further combine cascades with ensembles to significantly improve accuracy-compute trade-offs~\citep{Pedrajas:2005,wang2022wisdom}.

A key ingredient of cascades is the choice of deferral rule.
The simplest candidate
is to defer when the current model's \emph{confidence} in its prediction is sufficiently low.
Popular confidence measures include
the maximum predictive
probability over all classes~\citep{wang2022wisdom},
and the entropy of the predictive distribution~\citep{Huang:2018}. 
Despite being oblivious to the nature of the cascade
---
e.g., not modelling the errors of downstream models
---
such \emph{confidence-based deferral} works remarkably well in practice 
\citep{Wang:2018,Geifman:2019,Gangrade:2021,wang2022wisdom,Mamou:2022,narasimhan2022posthoc}.
Indeed, it has often been noted that such deferral can perform on-par with more complex modifications to the model training~\citep{Geifman:2019,Gangrade:2021,Kag:2023}.
However, the reasons for this success remain unclear;
further, it is not clear if there are specific practical settings where confidence-based deferral may perform poorly~\citep{narayan2022}.

In this paper, we initiate a systematic study of 
the potential limitations of confidence-based deferral for cascades.
Our findings and contributions are:
\begin{enumerate}[label=(\roman*),itemsep=0pt,topsep=0pt,leftmargin=16pt]
    \item We establish a novel result characterising the theoretically optimal deferral rule (Proposition~\ref{prop:oracle_k2}), which for a two-model cascade relies on the confidence of both model 1 and model 2. 
    
    \item In many regular classification tasks where model 2 gives a consistent estimate of the true posterior probability, confidence-based deferral is highly competitive.
    However, we show that in some settings,
    confidence-based deferral can be significantly sub-optimal both in theory and practice (\Cref{sec:chow-versus-oracle}). 
    This includes when
    (1) model 2's error probability is highly non-uniform across samples, which can
    happen when model 2 is a \emph{specialist} model, 
    (2) labels are subject to noise,
    and
    (3) there is distribution shift between the train and test set.
    
    \item Motivated by this, 
    we then study a series of \emph{post-hoc} deferral rules, that seek to mimic the form of the optimal deferral rule (\S\ref{sec:post_hoc}).
    We show that post-hoc deferral can significantly improve upon confidence-based deferral 
    in the aforementioned settings.
\end{enumerate}

To the best of our knowledge, this is the first work that precisely identifies 
specific practical problems settings where confidence-based deferral can be sub-optimal for cascades. 
Our findings give insights on when it is appropriate to deploy a confidence-based cascade model in practice.

\section{Background and Related Work}

Fix an instance space $\XCal$ and label space $\YCal = [ L ] \defEq \{ 1, 2, \ldots, L \}$,
and let $\Pr$ be a distribution over $\XCal \times \YCal$.
Given a training sample $S = \{ ( x_n, y_n ) \}_{n \in [N]}$ drawn from $\Pr$,
multi-class classification seeks a \emph{classifier} $h \colon \XCal \to {\YCal}$ with low \emph{misclassification error} $R( h ) = \Pr( y \neq h( x ) )$.
We may parameterise $h$ as
$h( x ) = \argmax_{{y}' \in \YCal} f_{{y}'}( x )$
for a \emph{scorer} $f \colon \XCal \to \Real^{L}$.
In neural models,
one expresses $f_y( x ) = w_y^\top \Phi( x )$ for class weights $w_y$ and embedding function $\Phi$.
It is common to construct a \emph{probability estimator} $p \colon \XCal \to \Delta^L$ using the softmax transformation, $p_y( x ) \propto \exp( f_y( x ) )$.

\subsection{Cascades and Deferral Rules}

Conventional neural models involve a fixed inference cost for any test sample $x \in \XCal$.
For large models, this cost may prove prohibitive.
This has motivated 
several approaches to \emph{uniformly} lower the inference cost for \emph{all} samples,
such as 
architecture modification~\citep{Larsson:2017,Veniat:2018,Huang:2018,Vu:2020},
stochastic depth~\citep{Huang:2016,Fan:2020},
network sparsification~\citep{Liu:2015},
quantisation~\citep{Hubara:2017},
pruning~\citep{LeCun:1989},
and distillation~\citep{Bucilua:2006,Hinton:2015,Rawat:2021}.
A complementary strategy is to 
\emph{adaptively} lower the inference cost for \emph{``easy''} samples,
while reserving the full cost only for ``hard'' samples~\citep{Han:2022,Zilberstein:1996,Scardapane:2020,Panda:2016,Teerapittayanon:2016,Wang:2018,Huang:2018,Schwartz:2020,Xin:2020,Elbayad:2020,Liu:2020,Zhou:2020,Xin:2020b}.

Cascade models
are a classic example of 
adaptive predictors,
which 
have proven useful in vision tasks such as object detection~\citep{Viola:2001},
and have grown increasingly popular in natural language processing~\citep{Mamou:2022,Varshney:2022,Khalili:2022,Dohan:2022}.
In the vision literature,
cascades are often designed for binary classification problems, 
with lower-level classifiers being used to quickly identify negative samples~\citep{Brubaker:2007,Zhang:2008,Saberian:2010,Chen:2012,DeSalvo:2015,Streeter:2018}.
A cascade is composed of two components: 
\begin{enumerate}[label=(\roman*), itemsep=0pt,topsep=0pt,leftmargin=16pt]
    \item a collection of base models (typically of non-decreasing inference cost)
    \item a \emph{deferral rule} 
    (i.e., a function that decides which model to use for each $x \in \XCal$).
\end{enumerate}

For any test input, 
one executes the first model, and employs the {deferral rule} to determine whether to terminate with the current model's prediction, or to invoke the next model;
this procedure is repeated until one terminates, or reaches the final model in the sequence.
Compared to using a single model, the goal of the cascade is to offer comparable predictive accuracy, but lower average inference cost.

\removefornow{
Typically, the deferral rule is based purely on the current model's predictions.
Formally, let $h^{(1)}, \ldots, h^{(K)}$ denote a sequence of $K$ classifiers, and $r \colon \XCal \to \{ 0, 1 \}$ a deferral rule.
Let $k_{\rm acc} \defEq \{ k \in [K - 1] \colon r( h^{(k)}( x ) ) = 0 \}$ be the set of indices where $r$ recommends to not defer (i.e., terminate).
Then, the cascade classifier $h^{\rm cas} \colon \XCal \to \YCal$ predicts
\begin{align*}
    h^{\rm cas}( x ) = h^{k^*}( x )
    \text{ where }
    k^* &\defEq \begin{cases} \min k_{\rm acc} & \text{ if } k_{\rm acc} \neq \emptyset \\ K & \text{else.} \end{cases} %
\end{align*}
} %

While the base models and deferral rules may be trained jointly~\citep{Trapeznikov:2013,Kag:2023}, 
we focus on a setting where
the base models are \emph{pre-trained and fixed}, and the goal is to train only the deferral rule.
This setting is practically relevant, 
as it is often desirable to re-use powerful models that involve expensive training procedures.
Indeed, a salient feature of cascades is their ability to leverage off-the-shelf models
and simply adjust the desired operating point 
(e.g., the rate at which we call a large model).

\subsection{Confidence-based Cascades}
A simple way to define a deferral rule 
is by thresholding a model's \emph{confidence} in its prediction.
\begin{wrapfigure}{R}{0.5\textwidth}
    \vspace{-\baselineskip}
    \begin{minipage}[t][5cm]{0.5\textwidth}
        \begin{algorithm}[H]
            \caption{Confidence-based cascades of $K$ classifiers}
            \begin{algorithmic}[1]
                \REQUIRE $K \ge 2$ classifiers: $p^{(1)}, \ldots, p^{(K)} \colon \XCal \to \Delta^L$, 
                thresholds $c^{(1)}, \ldots, c^{(K-1)} \in [0, 1]$
                \REQUIRE An input instance $x \in \XCal$
                \item[]
                \FOR {$k=1, 2, \ldots, K-1$} 
                \IF{$\max_{y'} p^{(k)}_{y'}(x) > c^{(k)}$} 
                \STATE Predict class $\hat{y} = \argmax_{y'} p^{(k)}_{y'}(x)$
                \STATE \textbf{break}
                \ENDIF
                \ENDFOR
                \STATE Predict class $\hat{y} = \argmax_{y'} p^{(K)}_{y'}(x)$
            \end{algorithmic}
            \label{algo:chow}
        \end{algorithm}
    \end{minipage}
\end{wrapfigure}
While there are several means of quantifying and improving such confidence~\citep{Shafer:2008,Guo:2017,Kendall:2017,Jiang:2018}, 
we focus on
the \emph{maximum predictive probability} $\gamma( x ) \defEq \max_{y'} p( y' \mid x )$.
Specifically, given $K$ trained models, 
a confidence-based cascade is formed by picking the first model that whose confidence 
$\gamma( x )$ is sufficiently high~\citep{wang2022wisdom,Ramaswamy:2018}.
This is made precise in~\Cref{algo:chow}. 

Forming a cascade in this manner is appealing because it does not require retraining of any models or the deferral rule.
Such a post-hoc approach has been shown to give good trade-off between
accuracy and inference cost \citep{wang2022wisdom}.
Indeed, it has often been noted that such deferral can perform on-par with more complex modifications to the model training~\citep{Geifman:2019,Gangrade:2021,Kag:2023}.
However, the reasons for this phenomenon are not well-understood; further, it is unclear if there are practical settings where confidence-based cascades are expected to underperform.

To address this, 
we now formally analyse confidence-based cascades, and explicate their limitations.

\section{Optimal Deferral Rules for Cascades}
\label{sec:deferral}
In this section, 
we derive the oracle (Bayes-optimal) deferral rule, which will
allow us to understand the effectiveness and limitations of confidence-based deferral (\Cref{algo:chow}).

\subsection{Optimisation Objective}
For simplicity, we consider a cascade of $K=2$ pre-trained classifiers $h^{(1)},h^{(2)}\colon \XCal \to \YCal$;
our analysis readily generalises to cascades with $K > 2$ models (see Appendix~\ref{sec:multi_model_cascades}).
Suppose that the classifiers
are based on probability estimators $p^{(1)}, p^{(2)}$,
where for $i\in\{1,2\}$,
$p_{y'}^{(i)}(x)$ estimates the probability
for class $y'$ given $x$,
and
$h^{(i)}(x)\defEq\arg\max_{y'}p_{y'}^{(i)}(x)$.
We do not impose any restrictions on the training procedure or performance of these classifiers. 
We seek to learn a deferral rule 
$r \colon \XCal \to \{ 0, 1 \}$
that can decide
whether $h^{(1)}$ should be used (with $r(x)=0$),
or $h^{(2)}$ (with $r(x)=1$). 

\textbf{Constrained formulation}. 
To derive the optimal deferral rule, we must first specify a target metric to optimise.
Recall that cascades offer a suitable balance between average inference cost, and predictive accuracy.
Let us assume without loss of generality that $h^{(2)}$ has higher computational cost,
and that invoking $h^{(2)}$ incurs a constant cost $c > 0$. 
We then seek to find a deferral rule $r$ that maximises the predictive accuracy,
while invoking $h^{(2)}$ sparingly. This can be formulated as the following constrained optimisation problem:
\begin{equation}
\min_{r}
 \Pr(y\neq h^{(1)}(x),r(x)=0) +
 \Pr(y\neq h^{(2)}(x),r(x)=1)
 \quad \text{subject to} \quad 
 \Pr(r(x)=1) \leq \tau,
 \label{eq:constrained_opt}
\end{equation}
for a deferral rate $\tau \in [0,1 ]$. %
Intuitively, the objective measures the misclassification error only on samples where the respective classifier's predictions are used;  the constraint ensures that $h^{(2)}$ is invoked on at most $\tau$ fraction of samples. 
It is straight-forward to extend this formulation to generic cost-sensitive variants of the predictive accuracy~\citep{Elkan:2001}.

\textbf{Equivalent unconstrained risk}. Using Lagrangian theory, one may translate $\eqref{eq:constrained_opt}$ into minimizing an equivalent unconstrained risk. Specifically, 
under mild distributional assumptions (see e.g.,  \citet{neyman1933ix}), we can show that for any deferral rate $\tau$, there exists a  deferral cost $c \ge 0$ such that solving $\eqref{eq:constrained_opt}$ is equivalent to minimising the following unconstrained risk:
\begin{align}
 & R(r;h^{(1)},h^{(2)}) = \Pr(y\neq h^{(1)}(x),r(x)=0)
+
 \Pr(y\neq h^{(2)}(x),r(x)=1)+c \cdot \Pr(r(x)=1).
 \label{eq:oracle_risk_k2}
\end{align}

\textbf{Deferral curves and cost-risk curves}. 
As $c$ varies, one may plot the misclassification error of the resulting cascade as a function of deferral rate. We will refer to this as the \emph{deferral curve}.  
One may similarly compute the  \emph{cost-risk curve} that traces out the optimal cascade risk  \eqref{eq:oracle_risk_k2} as a function of $c$. Both these are equivalent ways of assessing the overall quality of a cascade. %

In fact, the two curves have an intimate point-line duality. 
Any point $(\tau, E)$ in deferral curve space  – where $\tau$ denotes deferral rate, and $E$ denotes  error – can be mapped to the line ${ (c, c \cdot \tau + E) : c \in [0, 1] }$ in cost curve space. 
Conversely, any point $(c, R)$ in cost curve space – where $c$ denotes  cost, and $R$ denotes risk – can be mapped to the line ${ (d, R - C \cdot d) : d \in [0, 1] }$ in deferral curve space. This is analogous to the correspondence between ROC and cost-risk curves in classification \citep{drummond2006cost}.
Following prior work \citep{Bolukbasi:2017,Kag:2023}, our experiments use deferral curves to compare methods.

\subsection{The Bayes-Optimal Deferral Rule}
\label{sec:bayes_optimal}
The \emph{Bayes-optimal} rule,
which minimises the risk $R(r;h^{(1)},h^{(2)})$ over all possible deferral rules $r \colon \XCal \to \{ 0, 1 \}$,
is given below.

\begin{proposition}
\label{prop:oracle_k2} 
Let $\eta_{y'}(x)\defEq\Pr(y'|x)$.
Then,
the Bayes-optimal deferral rule for the risk
in (\ref{eq:oracle_risk_k2}) is:
\begin{align}
r^{*}(x) & =\boldsymbol{1}\left[\eta_{h^{(2)}(x)}(x)-\eta_{h^{(1)}(x)}(x)>c\right],
\label{eq:optimal_rule}
\end{align}
\end{proposition}

(Proof in \Cref{sec:proof}.)
Note that for a classifier $h$, $\eta_{h(x)}(x)=\Pr(y=h(x)|x)$ i.e.,
the probability that $h$ gives a correct prediction for $x$. 
Observe that the randomness here reflects the inherent stochasticity of the labels $y$ for an input $x$, i.e.,
the aleatoric uncertainty~\citep{Hullermeier:2021}.
For the purposes of evaluating the deferral curve, 
the key quantity is $\eta_{h^{(2)}(x)}(x)-\eta_{h^{(1)}(x)}(x)$,
i.e.,  the \emph{difference in the probability of correct prediction under each classifier}.
This is intuitive:
it is optimal to defer to $h^{(2)}$ 
if the expected reduction in misclassification error exceeds the cost of invoking $h^{(2)}$. 

The Bayes-optimal deferral rule in~\eqref{eq:optimal_rule} is a theoretical construct,
which relies on knowledge of the true posterior probability $\eta$.
In practice, one is likely to use an approximation to this rule.
To quantify the effect of such an approximation, we may consider the \emph{excess risk} or \emph{regret}
of an arbitrary deferral rule $r$ over $r^*$.
We have the following.

\begin{corollary}
\label{corr:excess-risk}
Let $\alpha( x ) \defEq \eta_{h_1(x)}(x) - \eta_{h_2(x)}(x) + c$.
Then, the excess risk for an arbitrary $r$ is
$$R( r; h^{(1)}, h^{(2)} ) - R( r^*; h^{(1)}, h^{(2)} ) = \mathbb{E}_x[ ( \boldsymbol{1}(r(x) = 1) - \boldsymbol{1}(\alpha(x) < 0) ) \cdot \alpha(x) ]. $$
\end{corollary}

Intuitively, the above shows that when we  make deferral decisions that disagree with the Bayes-optimal rule,
we are penalised proportional to the difference between the two models' error probability.

\subsection{Plug-in Estimators of the Bayes-Optimal Deferral Rule}
\label{sec:plug_in_rule}

In practice, we may seek to approximate the Bayes-optimal deferral rule $r^*$ in~\eqref{eq:optimal_rule} with an estimator $\hat{r}$.
We now present several \emph{oracle estimators},
which will prove useful in our subsequent analysis.

\textbf{One-hot oracle}.
Observe that $\eta_{y'}( x ) = \mathbb{E}_{y \mid x}\left[ \boldsymbol{1}[y = y'] \right]$.
Thus,
given a test sample 
$(x,  y)$, 
one 
may replace the expectation with the observed label 
$y$
to yield the
ideal estimator 
$\hat{\eta}_{y'}(x) = \boldsymbol{1}[ y' = y ]$.
This 
results in the rule
\begin{align}
\hat{r}_{01}(x) & \defEq \boldsymbol{1}\left[ \boldsymbol{1}[y=h^{(2)}(x)] - \boldsymbol{1}[y=h^{(1)}(x)] > c \right].
\label{eq:optimal_rule_onehot}
\end{align}
One intuitive observation is that for high $c$, 
this rule only defers samples with $y \neq h^{(1)}( x )$ but $y = h^{(2)}( x )$,
i.e., \emph{samples where the first model is wrong, but the second model is right}.
Unfortunately, this rule is impractical, since it depends on the label $y$. 
Nonetheless, it 
serves as an oracle to help understand what one can gain
if we knew exactly whether the downstream model makes an error.

\textbf{Probability oracle}.
Following a similar reasoning as $\hat{r}_{{\rm 01}}$, another estimator is given by
\begin{align}
\text{ }\hat{r}_{{\rm prob}}(x) & \defEq
\boldsymbol{1}\left[ p_y^{(2)}(x) - p_{y}^{(1)}(x) > c\right].
\label{eq:rule_prob_oracle}
\end{align}
Intuitively $ p_y^{(2)}(x)$ can be seen as a label-dependent 
correctness score of model 2 on an instance $x$.

\textbf{Relative confidence}.
The above oracles rely on the true label $y$.
A more practical plug-estimator is $\hat{\eta}_{h^{(i)}( x )}( x ) = \max_{y'} p^{(i)}_{y'}( x )$,
which simply uses each model's softmax probabilities.
The rationale for this rests upon the assumption that the probability
model $p^{(i)}$ is a consistent estimate of the true posterior probability
so that $\Pr(y|x)\approx p_{y}^{(i)}(x)$ for $i\in\{1,2\}$, where $(x,y)$ is
a labeled example. 
Thus,
$\eta_{h^{(i)}(x)}(x)=\Pr(h^{(i)}(x)|x)\approx p^{(i)}(h^{(i)}(x)|x)=\max_{y'}p_{y'}^{(i)}(x)$,
resulting in the rule 
\begin{align}
\text{ }\hat{r}_{{\rm rel}}(x) & \defEq1\left[\max_{y''}p_{y''}^{(2)}(x)-\max_{y'}p_{y'}^{(1)}(x)>c\right].\label{eq:rule_diff_conf}
\end{align}
Observe that this
deferral decision depends on the confidence
of \emph{both} models, in contrast to confidence-based deferral which relies only
on the confidence of the \emph{first} model.

Note that the above oracles cannot be used directly for adaptive computation, because  the second model is invoked on \emph{every} input.
Nonetheless, they  can inform us
about the available headroom to improve over confidence-based deferral
by
considering the confidence of the downstream model.
As shall be seen in \Cref{sec:conf_vs_oracle}, these estimators are useful for deriving objectives to train a post hoc deferral rule.

\subsection{Relation to Existing Work}
\label{sec:related_optimal}

The two-model cascade 
is closely connected to the literature on 
\emph{learning to defer to an expert}~\citep{Madras:2018,Mozannar:2020,charusaie22a}.
Here, the goal is to learn a base classifier $h^{(1)}$ that has the option of invoking an ``expert'' model $h^{(2)}$;
this invocation is controlled by a deferral rule $r$.
Indeed, the risk~\eqref{eq:oracle_risk_k2} is a special case of~\citet[Equation 2]{Mozannar:2020}, where the second model is considered to be an ``expert''.
Proposition~\ref{prop:oracle_k2} is a simple generalisation of~\citet[Proposition 2]{Mozannar:2020}, with the latter assuming $c = 0$.
In \Cref{sec:multi_model_cascades}, we generalise \Cref{prop:oracle_k2} to the cascades
of $K > 2$ models.

\section{From Confidence-Based to Post-Hoc Deferral}
\label{sec:conf_vs_oracle}

Having presented the optimal deferral rule in Proposition~\ref{prop:oracle_k2},
we now use it to explicate some 
failure modes for confidence-based deferral,
which will be empirically demonstrated in~\S\ref{sec:chow_suffices}.

\subsection{When Does Confidence-Based Deferral Suffice?}
\label{sec:chow-versus-oracle}

Suppose as before we have probabilistic models $p^{(1)}, p^{(2)}$.
Recall from Algorithm~\ref{algo:chow} that for constant $c^{(1)} > 0$,
confidence-based deferral employs
the rule
\begin{equation}
\label{eqn:confidence_deferral}
    \hat{r}_{\rm conf}( x ) = \boldsymbol{1}\left[ \max_{y'} p^{(1)}_{y'}( x ) < c^{(1)} \right].
\end{equation}
Following~\S\ref{sec:plug_in_rule},~\eqref{eqn:confidence_deferral} may be regarded as a
plug-in estimator for the ``population confidence'' rule 
$r_{\rm conf}( x ) \defEq \boldsymbol{1}\left[ \eta_{h^{(1)}}( x ) < c^{(1)} \right]$.
Contrasting this to Proposition~\ref{prop:oracle_k2},
we have the following:

\begin{lemma}
    \label{lemm:when-confidence-suffice}
    Assume that for any $x,x' \in \XCal$, $\eta_{h^{(1)}}( x ) \le \eta_{h^{(1)}}( x' )$ if and only if $\eta_{h^{(1)}}( x ) - \eta_{h^{(2)}}( x ) \le \eta_{h^{(1)}}( x' ) - \eta_{h^{(2)}}( x' ) $ (i.e., $\eta_{h^{(1)}}( x )$ and $\eta_{h^{(1)}}( x ) - \eta_{h^{(2)}}( x )$ produce the same ordering over instances $x \in \XCal$).
    Then,  the deferral rule 
    $r_{\rm conf}$
    produces the same \emph{deferral curve} as the Bayes-optimal rule~\eqref{eq:optimal_rule}.
    
\end{lemma}    

Lemma~\ref{lemm:when-confidence-suffice} studies the agreement between the \emph{deferral curves} of 
$r_{\rm conf}$ and the Bayes-optimal solution,
which eliminates the need for committing to a specific cost $c^{(1)}$.
The lemma has an intuitive interpretation:
population confidence-based deferral is optimal
if and only if the \emph{absolute} confidence in model 1's prediction agrees with the \emph{relative} confidence is model 1 versus model 2's prediction.

Based on this, we now detail some cases where confidence-based deferral succeeds or fails.

\textbf{Success mode: expert $h^{(2)}$}.
Lemma~\ref{lemm:when-confidence-suffice} has one immediate, intuitive consequence:
\emph{confidence-based deferral 
is optimal when
the downstream model has a constant error probability},
i.e., $\eta_{h_2(x)}( x )$ is a constant for all $x \in \XCal$.
This may happen, 
e.g.,
if that the labels are deterministic given the inputs, and the second classifier $h^{(2)}$ perfectly predicts them.
Importantly, note that this is a sufficient (but \emph{not} necessary) condition for the optimality of confidence-based deferral.

\textbf{Failure mode: specialist $h^{(2)}$}.
As a converse to the above,
one setting where confidence-based deferral may fail is when
when the downstream model is a \emph{specialist},
which performs well only on a particular sub-group of the data (e.g., a subset of classes).
Intuitively,
confidence-based deferral may 
\emph{erroneously forward samples where $h^{(2)}$ performs worse than $h^{(1)}$}.

Concretely, 
suppose there is a data sub-group  $\XCal_{\rm good} \subset \XCal$ where $h^{(2)}$ 
performs exceptionally well,
i.e.,
$\eta_{h^{(2)}(x)} \approx 1$ when $x \in \XCal_{\rm good}$.
On the other hand,
suppose
$h^{(2)}$ does not perform well on $\XCal_{\rm bad} \defEq \XCal \setminus \XCal_{\rm good}$,
i.e.,
$\eta_{h^{(2)}(x)} \approx 1/L$ when $x \in \XCal_{\rm bad}$.
Intuitively, 
while $\eta_{h^{(1)}(x)}$ may be relatively low for $x \in \XCal_{\rm bad}$,
it is strongly desirable to \emph{not} defer such examples,
as $h^{(2)}$ performs even worse than $h^{(1)}$;
rather, it is preferable to identify and defer samples $x \in \XCal_{\rm good}$.

\textbf{Failure mode: label noise}.
Confidence-based deferral can fail when there are high levels of label noise.
Intuitively, in such settings,
confidence-based deferral may \emph{wastefully forward samples where $h^{(2)}$ performs no better than $h^{(1)}$}.
{Concretely, 
suppose that instances $x \in \XCal_{\rm bad} \subset \XCal$
may be mislabeled as one from a different, random class.
For $x  \in \XCal_{\rm bad}$, regardless of how the two models  $h^{(1)}, h^{(2)}$ perform, we have 
$\eta_{h^{(1)}(x)}(x), \eta_{h^{(2)}(x)}(x) = 1/L$ (i.e., the accuracy of classifying these instances is chance level in expectation). Since $\eta_{h^{(1)}(x)}(x)$ is low, confidence-based deferral will tend to
defer such input instance $x$. 
However, this is a sub-optimal decision since model 2 is more computationally expensive, and expected to have the same chance-level performance. 
}

\textbf{Failure mode: distribution shift}.
Even when model $h^{(2)}$ is an expert model, an intuitive setting where confidence-based deferral can fail is if there is \emph{distribution shift} between the train and test $\Pr( y \mid x )$~\citep{Quinonero-Candela:2009}.
In such settings, even if $p^{(1)}$ produces reasonable estimates of the \emph{training} class-probability, these may translate poorly to the test set.
There are numerous examples of confidence degradation under such shifts,
such as the presence of \emph{out-of-distribution} samples~\citep{Nguyen:2015,Hendrycks:2017},
and the presence of a \emph{label skew} during training~\citep{Wallace:2012,Samuel:2020}.
We shall focus on the latter in the sequel.

\begin{table*}[!t]
    \centering
    \renewcommand{\arraystretch}{1.25}
    
    \resizebox{\linewidth}{!}{
    \begin{tabular}{@{}llllp{1in}@{}}
        \toprule
        \textbf{Training label $z$} & \textbf{Loss} & \textbf{Deferral rule} & \textbf{Method label}&  \textbf{Comment} \\
        \toprule
        \toprule
         $ 1[y=h^{(2)}(x)] - 1[y=h^{(1)}(x)] $ & Mean-squared error &  $g(x) > c$  & \texttt{Diff-01} &   Regression  \\
         $  p^{(2)}_{y}(x) - p^{(1)}_{y}(x) $ & Mean absolute error &   $g(x) > c$  & \texttt{Diff-Prob} &  Regression \\
         $\max_{y'} p^{(2)}_{y'}(x) $ & Mean-squared error&  $g(x) - \max_{y'} p^{(1)}_{y'}(x) > c$ &  \texttt{MaxProb} & Regression \\         
        \bottomrule
    \end{tabular}
    }

    \caption{Candidate post-hoc estimators of the oracle rule in \eqref{eq:oracle_risk_k2}.
    We train a post-hoc model $g(x)$ on a training set $\{(x_i, z_i)\}_{i=1}^n$ so as
    to predict the label $z$. 
    Here, $(x,y) \in \XCal \times \YCal$ 
    is a labeled example. 
    }
    \label{tbl:post_hoc_estimates}
\end{table*}

\subsection{Post-Hoc Estimates of the Deferral Rule}
\label{sec:post_hoc}

Having established that 
confidence-based deferral may be sub-optimal in certain settings,
we now consider the viability of deferral rules that are \emph{learned} in a post-hoc manner.
Compared to confidence-based deferral, such rules aim to explicitly account for \emph{both} the confidence of model 1 and 2, 
and thus avoid the failure cases identified above.

The key idea behind such post-hoc rules is to directly mimic the optimal deferral rule in~\eqref{eq:optimal_rule}.
Recall that 
this optimal rule
has a dependence on the output of $h^{(2)}$;
unfortunately, querying $h^{(2)}$ defeats the entire purpose of cascades. 
Thus, our goal is to estimate~\eqref{eq:optimal_rule} using only
the outputs of $p^{(1)}$.

We summarise a number of post-hoc estimators in \Cref{tbl:post_hoc_estimates},
which are directly
motivated by the One-hot, Probability, and Relative Confidence Oracle respectively from \Cref{sec:plug_in_rule}.
The first is to 
learn when model 1 is incorrect, and model 2 is correct.
For example, 
given a validation set,
suppose we construct
samples 
$S_{\rm{val}} \defEq \{ ( x_i, z^{(1)}_i ) \}$,
where $z_i^{(1)} = 1[ y = h^{(2)}( x_i )] - 1[y = h^{(1)}( x_i ) ]$.
Then,
we fit 
$$ \min_{g \colon \XCal \to \Real} \frac{1}{| S_{\rm{val}} |} \sum_{( x_i, z^{(1)}_i ) \in S_{\rm val}} \ell( z^{(1)}_i, g( x_i ) ), $$
where, e.g., $\ell$ is the square loss.
The score $g( x )$ may be regarded as the confidence in deferring to model 2.
Similarly, 
a second approach is to
perform regression to predict $z^{(2)}_i = p_y^{(2)}(x_i) - p_y^{(1)}(x_i)$. 
The third approach is to
directly estimate 
$z^{(3)}_i = \max_{y'} p^{(2)}_{y'}( x_i )$ 
using predictions of the first model.

\label{sec:post_hoc_suffice}

As shall be seen in~\S\ref{sec:chow_suffices}, such post-hoc rules can learn to avoid the failure cases for confidence-based deferral identified in the previous section.
However, it is important to note that there are some conditions where such rules may not offer benefits over confidence-based deferral.

\textbf{Failure mode: Bayes $p^{(2)}$}.
Suppose that the model $p^{(2)}$ exactly matches the Bayes-probabilities, i.e.,
$p^{(2)}_{y'}( x ) = \Pr( y' \mid x )$.
Then, estimating $\max_{y'} p^{(2)}_{y'}( x )$ is equivalent to estimating $\max_{y'} \Pr( y' \mid x )$.
However, the goal of model 1 is \emph{precisely} to estimate $\Pr( y \mid x )$.
Thus, if $p^{(1)}$ is sufficiently accurate, in the absence of additional information (e.g., a fresh dataset), it is unlikely that one can obtain a better estimate of this probability than that provided by $p^{(1)}$ itself.
This holds even if the $\Pr( y \mid x )$ is non-deterministic, and so the second model has non-trivial error.

\textbf{Failure mode: non-predictable $p^{(2)}$ error}.
When the model $p^{(2)}$'s outputs are not strongly predictable, 
post-hoc deferral may devolve to regular confidence-based deferral.
Formally, suppose we seek to predict $z \defEq \max_{y'} p^{(2)}_{y'}( x )$, e.g., as in {\tt MaxProb}.
A non-trivial predictor must achieve an average square error smaller than the variance of $z$, i.e.,
$\mathbb{E}[ ( z - \mathbb{E}[ z ] )^2 ]$.
If $z$ is however not strongly predictable, 
the estimate will be tantamount to simply using the constant $\mathbb{E}[ z ]$.
This brings us back to the assumption of model 2 having a constant probability of error,
i.e., confidence-based deferral.

\subsection{Finite-Sample Analysis for Post-Hoc Deferral Rules}
We now formally quantify the gap in performance between the learned post-hoc rule and the Bayes-optimal rule $r^*  = 1[g^*(x) > c]$ in Proposition~\ref{prop:oracle_k2}, where  $g^*(x) = \eta_{h^{(2)}(x)}(x)-\eta_{h^{(1)}(x)}(x)$. 
We will consider the case where the validation sample $S_{\rm val}$ is constructed with labels $z^{(1)}_i$. We pick a scorer $\hat{g}$ that minimises the average squared loss 
$\frac{1}{| S_{\rm{val}} |} \sum_{( x_i, z_i ) \in S_{\rm val}} ( z_i - g( x_i ) )^2$ 
over a hypothesis class $\mathcal{G}$. 
We then construct a deferral rule  $\hat{r}(x) = 1[\hat{g}(x) > c]$. 

\begin{lemma}
\label{lem:finite-sample}
Let   $\mathcal{N}(\mathcal{G}, \epsilon)$ denote the covering number  of $\mathcal{G}$ with the $\infty$-norm. 
Suppose for any $g \in \mathcal{G}$,  $(z - g(x))^2 \leq B, \forall (x, z)$. 
Furthermore, let $\tilde{g}$ denote the minimizer of the population squared loss $\E\left[(z - g(x))^2\right]$ over $\mathcal{G}$, where $z = 1[ y = h^{(2)}( x )] - 1[y = h^{(1)}( x ) ]$. Then 
for any $\delta \in (0,1)$, with probability at least $1 - \delta$ over draw of $S_{\rm val}$, the excess risk for $\hat{r}$ is bounded by
\begin{align*}
\lefteqn{R(\hat{r};h^{(1)},h^{(2)}) - R(r^*;h^{(1)},h^{(2)})}\\[-2pt]
&\leq\,
2\cdot \Bigg(
\underbrace{
    \E_x\left[ (\tilde{g}(x)  -  g^*(x))^2\right]}_{\text{Approximation error}}  + 
    \underbrace{4\cdot\inf_{\epsilon > 0}
    \left\{
        B\sqrt{
            \frac{ 2\cdot\log \mathcal{N}(\mathcal{G}, \epsilon) }{|S_{\rm val}|}
        }
    \right\}
    \,+\,
    \mathcal{O}\left(\sqrt{\frac{\log(1/
    \delta)}{|S_{\rm val}|}}\right)}_{\text{Estimation error}}
\Bigg)^{1/2}.
\end{align*}
\end{lemma}

We provide the proof in \S\ref{app:finite-sample}.
The first term on the right-hand side (RHS) is an irreducible approximation error, quantifying the distance between the best possible model
in the class to the oracle scoring function. 
The second and the third terms on RHS quantify total estimation error.

\subsection{Relation to Existing Work}
As noted in~\Cref{sec:related_optimal}, learning a deferral rule for a two-model cascade is closely related to existing literature in learning to defer to an expert.
This in turn is a generalisation of the classical literature on \emph{learning to reject}~\citep{Hendrickx:2021,cao2022}, which refers to classification settings where one is allowed to abstain from predicting on certain inputs.
The population risk here is a special case of~\eqref{eq:oracle_risk_k2},
where $h^{(2)}( x )$ is assumed to perfectly predict $y$.
The resulting Bayes-optimal classifier is known as Chow's rule~\citep{Cho1970,Ramaswamy:2018},
and exactly coincides with the deferral rule in 
Lemma~\ref{lemm:when-confidence-suffice}.
Plug-in estimates of this rule are thus analogous to confidence-based deferral,
and have been shown to be similarly effective~\citep{Ni:2019}.

In settings where one is allowed to modify the training of $h^{(1)}$, it is possible to construct losses that jointly optimise for both $h^{(1)}$ and $r$~\citep{Bartlett:2008,CorDeSMoh2016,Ramaswamy:2018,Thulasidasan:2019,Charoenphakdee:2021,Gangrade:2021,Kag:2023}.
While effective, these are not applicable in our setting involving pre-trained, black-box classifiers.
Other variants of post-hoc methods have been considered in~\citet{narasimhan2022posthoc}, and implicitly in~\citet{Trapeznikov:2013}; 
however, here we more carefully study the different possible ways of constructing these methods, and highlight when they may fail to improve over confidence-based deferral.

\section{Experimental Illustration}
\label{sec:experiments}
{In this section, we provide empirical evidence to support our analysis in \Cref{sec:chow-versus-oracle} by considering the three failure modes in which confidence-based deferral underperforms. 
For each of these settings, we compute \emph{deferral curves} that plot the classification accuracy versus the fraction of samples deferred to the second model (which implicitly measures the overall compute cost).
In line with our analysis in \Cref{sec:chow-versus-oracle}, post-hoc deferral rules offer better accuracy-cost trade-offs in these settings. 
}

\begin{figure*}[!t]

    \resizebox{0.99\linewidth}{!}{%
        \subcaptionbox{Standard ImageNet\\(100\% non-dog training images)}{
          \includegraphics[scale=0.3]{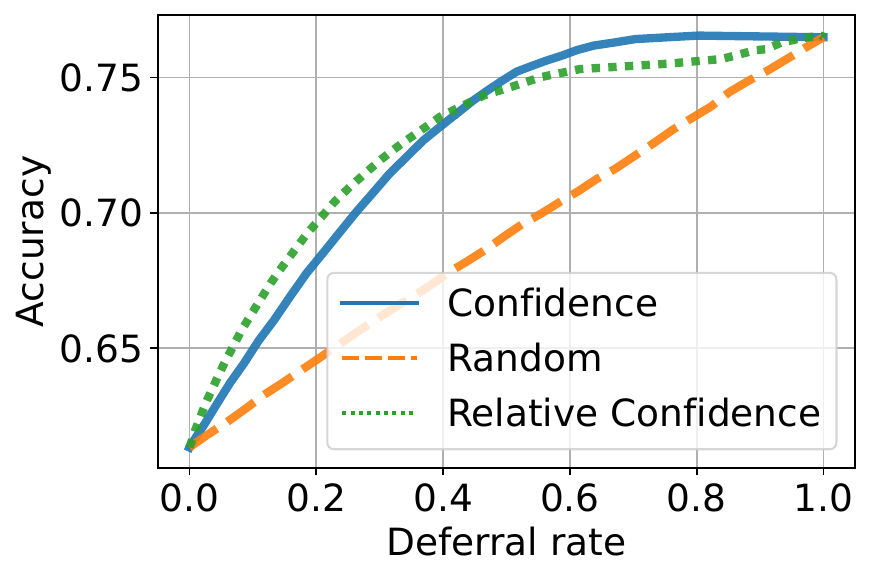}
        }
        \subcaptionbox{4\% non-dog training images}{
          \includegraphics[scale=0.3]{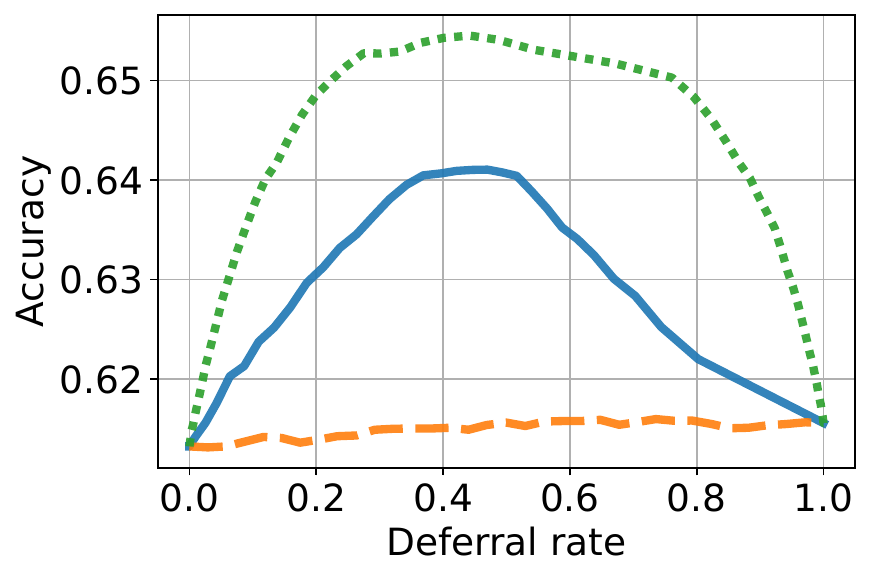}
        }  
        \subcaptionbox{2\% non-dog training images}{
          \includegraphics[scale=0.3]{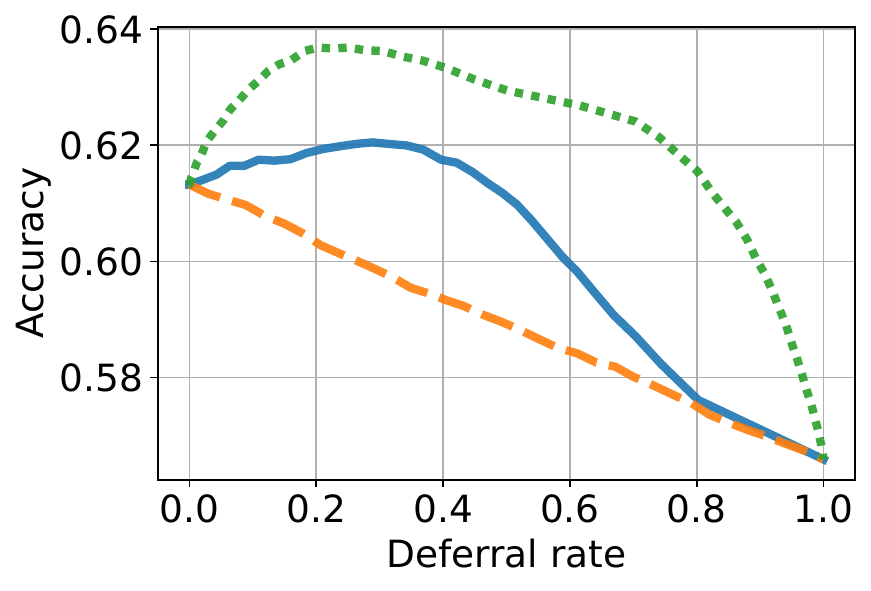}
        }          
    }
    
    \caption{Test accuracy vs deferral rate of plug-in estimates (\Cref{sec:plug_in_rule}) for the oracle rule.
     Here, $h^{(1)}$ is a MobileNet V2 trained on all ImageNet classes, and $h^{(2)}$ is a dog specialist trained on all images in the dog synset plus a fraction of non-dog training examples, which we vary.
     As the fraction decreases, $h^{(2)}$ specialises in classifying different types of dogs. By considering the confidence of $h^{(2)}$ (Relative Confidence), one gains accuracy by selectively deferring only dog images.
    }
    \vspace{-2mm}
    \label{fig:oracle_curves}
\end{figure*}

\subsection{Confidence-Based versus Oracle Deferral}
\label{sec:oracle_experiment}

We begin by
illustrating the benefit of considering confidence of the second model when constructing a deferral rule.
In this experiment, $h^{(1)}$ is a generalist (i.e., trained on all ImageNet classes), and $h^{(2)}$ is a dog specialist trained on all images in the dog synset, plus a fraction of non-dog training examples, which we vary. There are 119 classes in the dog synset. We use  MobileNet V2 \citep{sandler2018mobilenetv2} as $h^{(1)}$, and a larger EfficientNet B0 \citep{tan2019efficientnet} as $h^{(2)}$.  For hyperparameter details, see \Cref{app:hyperparams}.

\Cref{fig:oracle_curves} shows the accuracy of confidence-based deferral (Confidence) and Relative Confidence (\Cref{eq:rule_diff_conf}) on the standard ImageNet test set as a function of the deferral rate. We realise different deferral rates by varying the value of the deferral threshold $c$.
In \Cref{fig:oracle_curves}a, the fraction of non-dog training images is 100\% i.e., model 2 is also a generalist trained on all images. In this case, we observe that Relative Confidence offers little gains over Confidence. 

However, in \Cref{fig:oracle_curves}b and \Cref{fig:oracle_curves}c, 
as the fraction of non-dog training images decreases, the 
non-uniformity of  $h^{(2)}$'s error probabilities increases i.e., $h^{(2)}$ starts to specialise to dog images. 
In line with our analysis in \Cref{sec:chow-versus-oracle}, confidence-based deferral underperforms when model 2's error probability is highly non-uniform.
That is, being oblivious to the fact that  $h^{(2)}$ specialises in dog images,
confidence-based deferral may erroneously defer non-dog images to it.
By contrast, accounting for model 2's confidence, as done by Relative Confidence, shows significant gains.

\subsection{Confidence-Based versus Post-Hoc Deferral}
\label{sec:chow_suffices}

From~\S\ref{sec:oracle_experiment},
one may construct better deferral rules by querying model 2 for its confidence.
Practically, however, querying model 2 at inference time defeats the entire purpose of cascades.
To that end,
we now 
compare 
confidence-based deferral
and 
the post-hoc estimators (\Cref{tbl:post_hoc_estimates}),
which 
do \emph{not} need to invoke model 2 at inference time.
We consider
each of the settings from \Cref{sec:chow-versus-oracle},
and demonstrate that post-hoc deferral can significantly outperform confidence-based deferral.
We present more experimental results in \Cref{app:more_results}, 
where we illustrate post-hoc deferral rules for $K > 2$ models.

\paragraph{Post hoc model training.}
For a post-hoc approach to be practical, the overhead from invoking a post-hoc model must be small relative to the costs of $h^{(1)}$ and $h^{(2)}$.
To this end, in all of the following experiments, the post-hoc model $g\colon \XCal \to \mathbb{R}$
is based on a lightweight, three-layer Multi-Layer Perceptron (MLP) that takes as input the 
probability outputs from model 1. That is, $g(x) = \mathrm{MLP}(p^{(1)}(x))$
where $p^{(1)}(x) \in \Delta_L$ denotes all probability outputs from model 1.
Learning $g$ amounts to learning the MLP as the two base models are fixed.
We train $g$ on a held-out validation set. For full technical details of the post-hoc model architecture and training, see \Cref{app:hyperparams}. We use the objectives described in \Cref{tbl:post_hoc_estimates} to train $g$.

\newcommand{\resizecap}[1]{{\normalsize #1}}
\begin{figure*}[!t]
    \resizebox{0.99\linewidth}{!}{%
        \subcaptionbox{\resizecap{Standard ImageNet.}}{
          \includegraphics[scale=0.33]{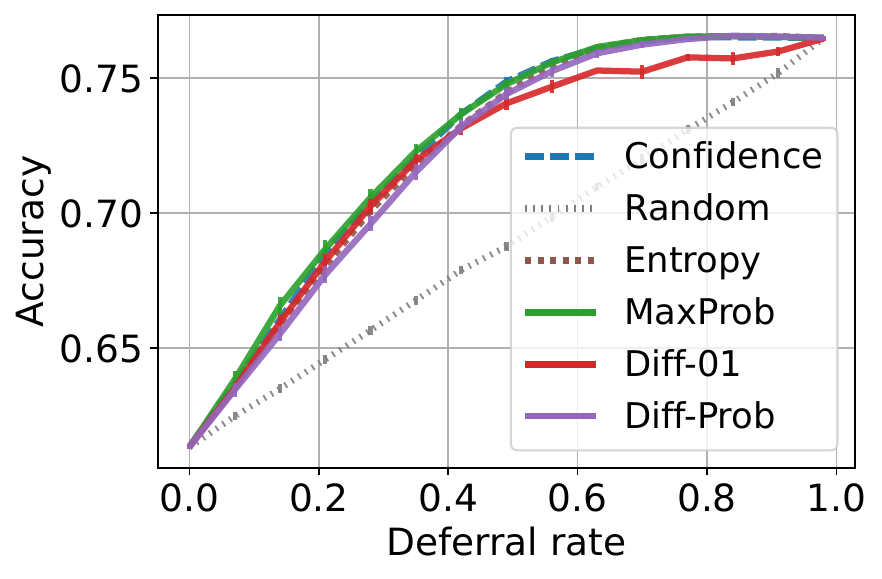}
        }
        \subcaptionbox{\resizecap{4\% non-dog training images.}}{
          \includegraphics[scale=0.33]{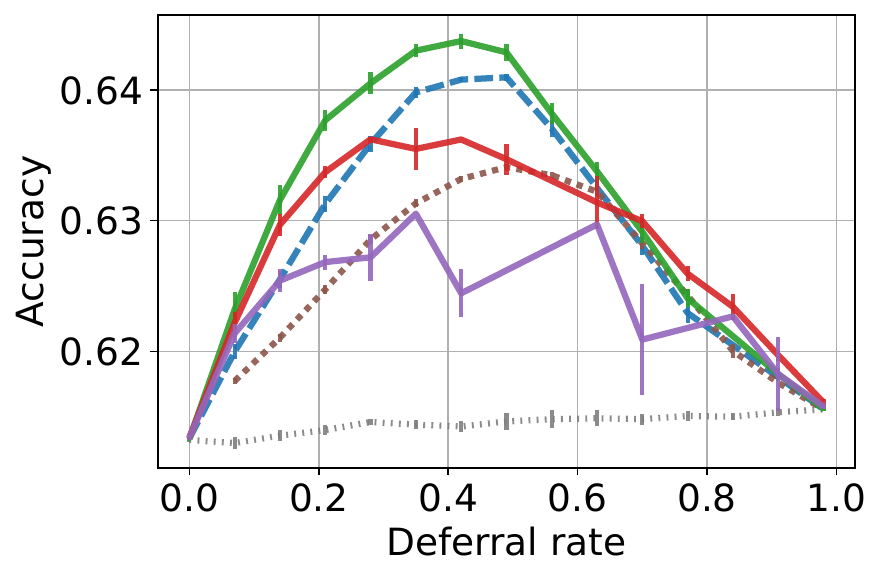}
        }
        \subcaptionbox{\resizecap{2\% non-dog training images.}}{
          \includegraphics[scale=0.33]{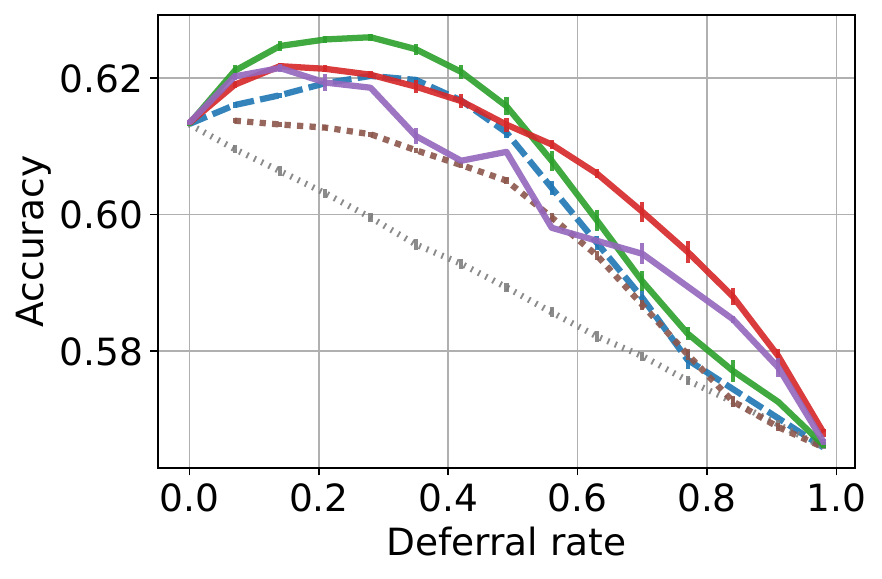}
        }          
    }
    \resizebox{0.99\linewidth}{!}{%
        \subcaptionbox{\resizecap{Standard CIFAR 100.}}{
          \includegraphics[scale=0.33]{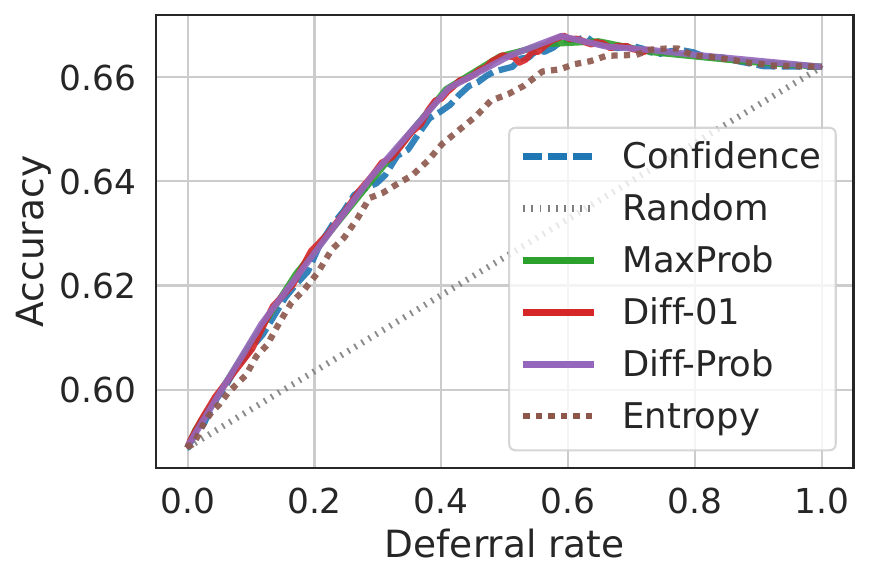}
        }
        \subcaptionbox{\resizecap{Label noise on first 10 classes.}}{
          \includegraphics[scale=0.33]{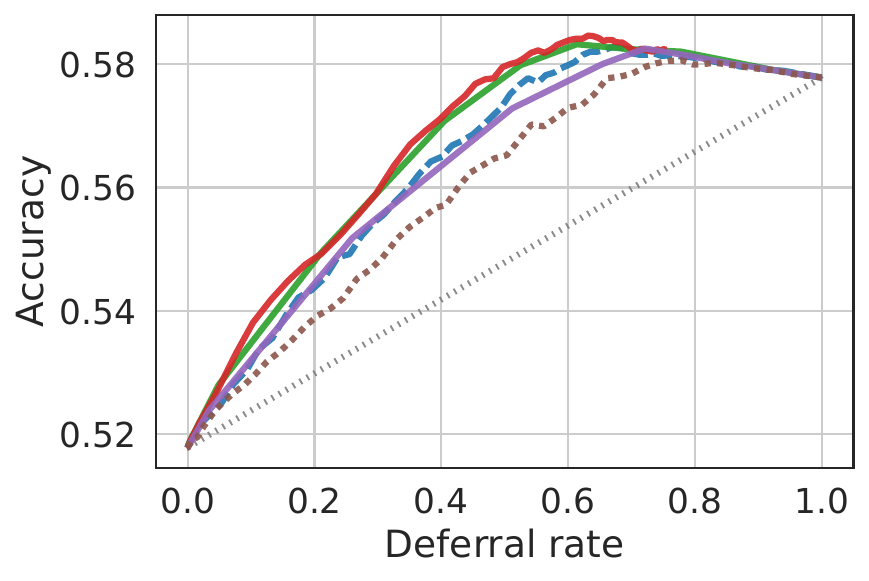}
        }
        \subcaptionbox{\resizecap{Label noise on first 25 classes.}}{
          \includegraphics[scale=0.33]{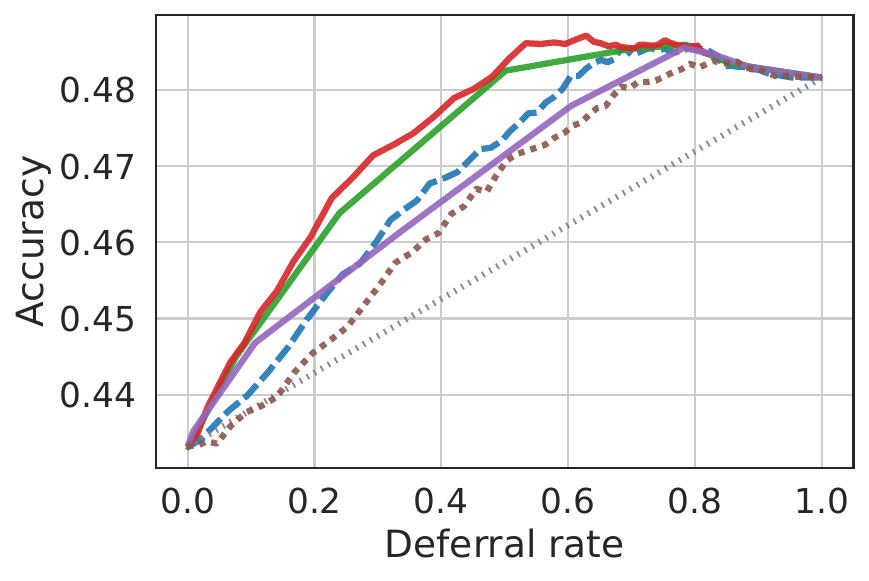}
        }        
    }
    \resizebox{0.99\linewidth}{!}{%
        \subcaptionbox{\resizecap{Standard CIFAR 100.}}{
          \includegraphics[scale=0.33]{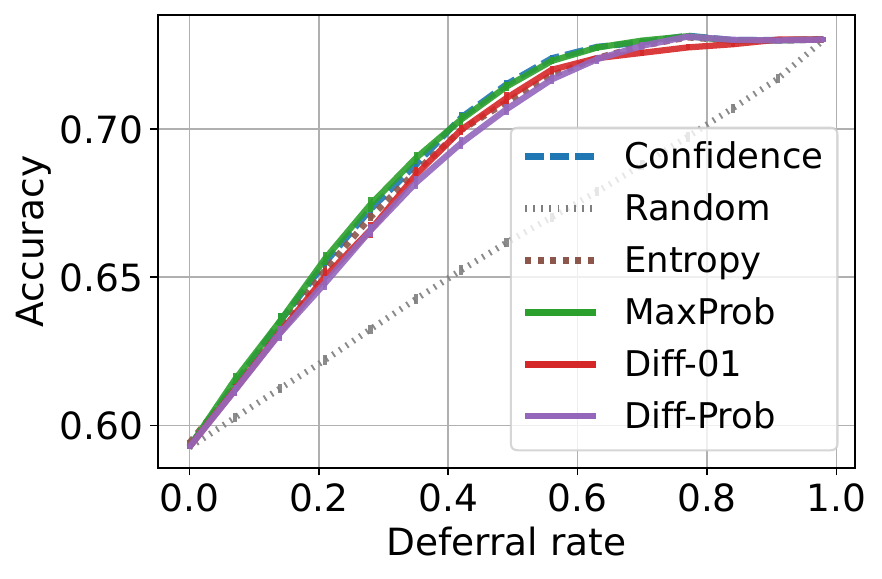}
        }
        \subcaptionbox{\resizecap{Head classes: 50.}}{
          \includegraphics[scale=0.33]{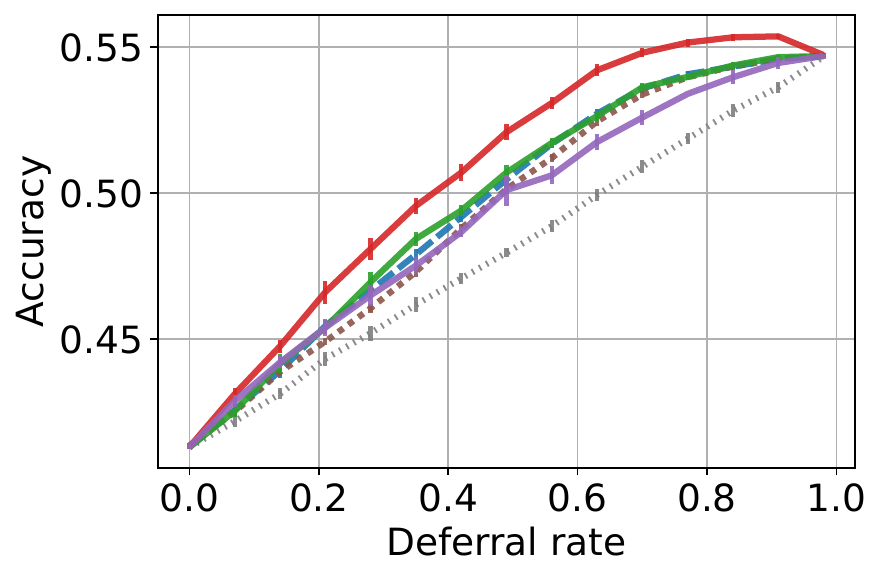}
        }  
        \subcaptionbox{\resizecap{Head classes: 25.}}{
          \includegraphics[scale=0.33]{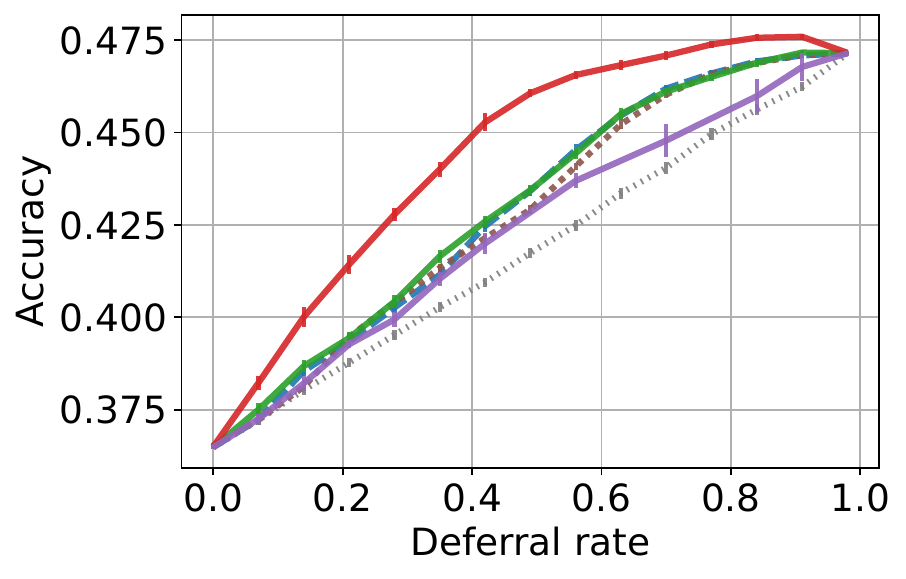}
        }          
    }
  
    \caption{Test accuracy vs deferral rate of the post-hoc approaches in \Cref{tbl:post_hoc_estimates} under the three settings described in \Cref{sec:chow-versus-oracle}: 1) specialist (row 1), 2) label noise (row 2), and 3) distribution shift.
    \textbf{Row 1}: As the fraction of non-dog training images decreases, model 2
    becomes a dog specialist model. Increase in the non-uniformity in its error probabilities allows post-hoc approaches to learn to only defer dog images.
        \textbf{Row 2}: As label noise increases, the difference in the probability of correct prediction under each model becomes zero (i.e., probability tends to chance level). Thus, it is sub-optimal to defer affected inputs since model 2's correctness is also at chance level.
    Being oblivious to model 2, confidence-based deferral underperforms. 
    For full details, see \Cref{sec:chow_suffices}.
    \textbf{Row 3}: As the skewness of the label distribution increases, so does the difference in the probability of correct prediction under each model (recall the optimal rule in \Cref{prop:oracle_k2}), and it becomes necessary to account for
    model 2's probability of correct prediction when deferring. Hence, confidence-based deferral underperforms.
    }
    \label{fig:three_settings_grid}
    \vspace{-2mm}
\end{figure*}

\textbf{Specialist setting}.
We start with the same ImageNet-Dog specialist setting used in \Cref{sec:oracle_experiment}. 
This time, we compare six methods: 
Confidence, Random, {\tt MaxProb}, {\tt Diff-01}, {\tt Diff-Prob}, and Entropy.
Random is a baseline approach that defers to either model 1 or model 2 at random;
{\tt MaxProb}, {\tt Diff-01}, and {\tt Diff-Prob} are the post-hoc rules described in \Cref{tbl:post_hoc_estimates};
Entropy defers based on the thresholding the entropy of $p^{(1)}$ \citep{Huang:2018,Kag:2023},
as opposed to the maximum probability.

Results for this setting are presented in \Cref{fig:three_settings_grid} (first row).
We see that there are gains from post-hoc deferral, especially in the low deferral regime:
it can accurately determine whether the second model is likely to make a mistake. 
Aligning with our analysis, 
for the generalist setting (\Cref{fig:three_settings_grid}a), 
confidence-based deferral is highly competitive, 
since $h^{(2)}$ gives a consistent estimate of $\Pr(y\mid x)$.

\textbf{Label noise setting}.
In this setting,
we look at a problem with label noise. We consider CIFAR 100 dataset where
training examples from pre-chosen $L_{\rm noise} \in \{0, 10, 25\}$ classes are assigned a uniformly drawn label. The case of $L_{\rm noise} = 0$ corresponds to the standard CIFAR 100 problem. 
We set $h^{(1)}$ to be CIFAR ResNet 8 and set $h^{(2)}$ to be CIFAR ResNet 14,
and train both models on the noisy data.
The results are shown
in \Cref{fig:three_settings_grid} (second row).
It is evident that when there is label noise, post-hoc approaches yield higher accuracy than confidence-based on a large range of deferral rates, aligning with our analysis in \Cref{sec:chow-versus-oracle}. 
Intuitively, confidence-based deferral tends to forward noisy samples to $h^{(2)}$, which performs equally poorly, thus leading to a waste of deferral budget.
By contrast, post-hoc rules can learn to ``give up'' on samples with extremely low model 1 confidence.

\textbf{Distribution shift setting}.
To simulate distribution shift, 
we consider a long-tailed version of CIFAR 100 \citep{krizhevsky2009learning} where
there are $h \in \{100, 50, 25\}$ head classes, and $100-h$ tail classes. Each head class has 500 training images, and each tail class has 50 training images. The standard CIFAR 100 dataset corresponds to $h=100$.
Both models $h^{(1)}$ (CIFAR ResNet 8) and $h^{(2)}$ (CIFAR ResNet 56) are trained
on these long-tailed datasets. 
At test
time, all methods are evaluated on the standard CIFAR 100 balanced test set, 
resulting in a label distribution shift.

We present our results in \Cref{fig:three_settings_grid} (third row).
In \Cref{fig:three_settings_grid}g, there is no distribution shift.
As in the case of the specialist setting, there is little to no gain from post-hoc approaches in this case since both models are sufficiently accurate.
As $h$ decreases from 100 to 50 (\Cref{fig:three_settings_grid}h) and 25 (\Cref{fig:three_settings_grid}i), there is more distribution shift at test time, and post-hoc approaches (notably {\tt Diff-01}) show clearer gains. 
To elaborate, the two base models are of different sizes and respond to the 
distribution shift differently, with CIFAR ResNet 56 being able to better handle
tail classes overall. {\tt Diff-01} is able to identify the superior performance of $h^{(2)}$ and defer input instances from tail classes.

\begin{figure*}[!t]
    \centering
    \resizebox{0.85\linewidth}{!}{%
        \subcaptionbox{Training performance}{
          \includegraphics[width=0.4\textwidth]{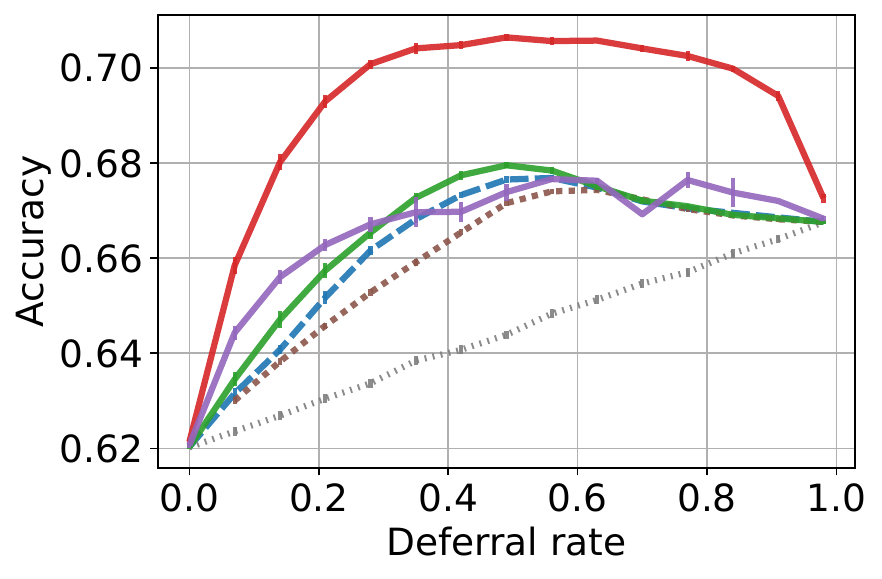}
        } 
        \subcaptionbox{Test performance}{
          \includegraphics[width=0.4\textwidth]{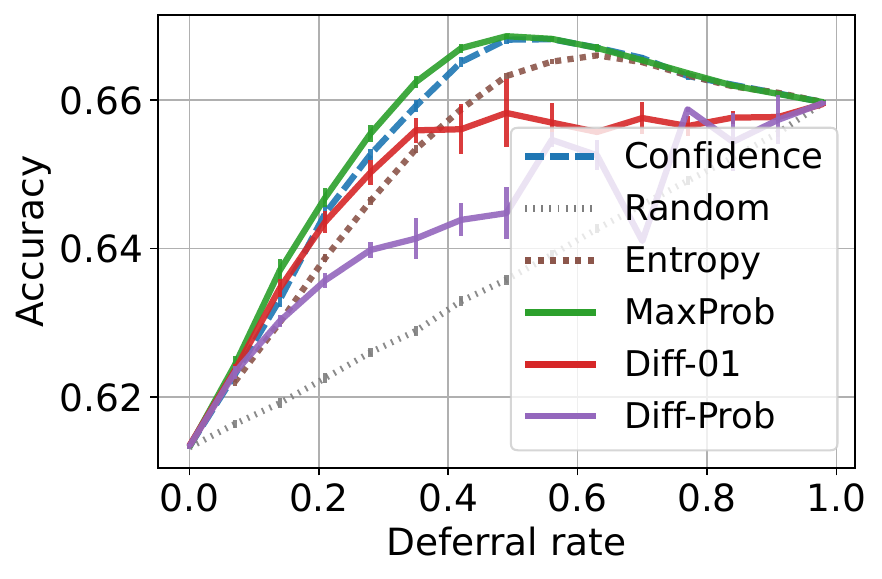}
        }  
    }
    \caption{Training and test accuracy of post-hoc approaches in the ImageNet-Dog specialist setting. Model 2 (EfficientNet B0) is trained with all dog images and 8\% of non-dog images. Observe that a post-hoc model  (i.e., {\tt Diff-01}) can severely overfit to the training set and fail to generalise.
    }
    \label{fig:posthoc_train_test_curves}
\end{figure*}

\subsection{On the Generalisation of Post-Hoc Estimators}
Despite the benefits of post-hoc approaches as demonstrated earlier, care must be taken in controlling
the capacity of the post-hoc models. 
We consider the same ImageNet-Dog specialist setting as in the top row of \Cref{fig:three_settings_grid}.
Here, model 2 is trained on all dog images, and a large fraction of non-dog images (8\%). 
Since model 2 has access to a non-trivial fraction of non-dog images,
the difference in the probability of correct prediction of the two models is less predictable.
We report deferral curves on both training and test splits in \Cref{fig:posthoc_train_test_curves}. 
Indeed, we observe that the post-hoc method {\tt Diff-01} can overfit,
and fail to 
generalise. Note that this is despite using a feedforward network with two hidden layers of only 64 and 16 units (see \Cref{app:hyperparams} for details on hyperparameters) to control the capacity of the post-hoc model.
Thoroughly investigating approaches to increase generalisation of post-hoc models will be an interesting topic for future study.

\section{Conclusion and Future Work}

The Bayes-optimal deferral rule we present suggests that key to optimally defer is to identify
when the first model is wrong and the second is right. Based on this result, we then study 
a number of estimators (\Cref{tbl:post_hoc_estimates}) to construct trainable post hoc deferral rules, and show that they can improve upon the commonly used confidence-based deferral.

While we have identified conditions under which confidence-based deferral underperforms (e.g., specialist setting, label noise), these are not exhaustive.
An interesting direction for future work is to design post-hoc deferral schemes attuned for settings involving other forms of distribution shift,
such as the presence of out-of-distribution samples.
It is also of interest to study the efficacy of more refined confidence measures, such as those based on conformal prediction~\citep{Shafer:2008}.
Finally, while our results have focussed on image classification settings, it would be of interest to study analogous trends for natural language processing models.

\bibliographystyle{plainnat}
\bibliography{ref}

\clearpage
\newpage

\addcontentsline{toc}{section}{Appendix} %
\part{} %

\begin{center}
{\LARGE \ourtitle{}}{\LARGE\par}
\par\end{center}

\begin{center}
\textcolor{black}{\Large{}Appendix}{\Large\par}
\par\end{center}

\parttoc %

\appendix

\section{Proofs}
\label{sec:proof}

\subsection{Proof of Proposition \ref{prop:oracle_k2}}
\begin{proof}[Proof of Proposition \ref{prop:oracle_k2}]

The risk in (\ref{eq:oracle_risk_k2}) can be written as
\begin{align*}
 & R(r;h^{(1)},h^{(2)})\\
 & =\Pr(y\neq h^{(1)}(x),r(x)=0)+\Pr(y\neq h^{(2)}(x),r(x)=1)+c\cdot\Pr(r(x)=1)\\
 & =\mathbb{E}\left[1[y\neq h^{(1)}(x)]\cdot1[r(x)=0]+1[y\neq h^{(2)}(x)]\cdot1[r(x)=1]+c1[r(x)=1]\right]\\
 & =\mathbb{E}\left[1[y\neq h^{(1)}(x)]\cdot\left(1-1[r(x)=1]\right)+1[y\neq h^{(2)}(x)]\cdot1[r(x)=1]+c1[r(x)=1]\right]\\
 & =\Pr(y\neq h^{(1)}(x))+\mathbb{E}\left[-1[y\neq h^{(1)}(x)]\cdot1[r(x)=1]+1[y\neq h^{(2)}(x)]\cdot1[r(x)=1]+c1[r(x)=1]\right]\\
 & =\Pr(y\neq h^{(1)}(x))+\mathbb{E}_{x}1[r(x)=1]\mathbb{E}_{y|x}\left[1[y\neq h^{(2)}(x)]-1[y\neq h^{(1)}(x)]+c\right]\\
 & =\Pr(y\neq h^{(1)}(x))+\mathbb{E}_{x}1[r(x)=1]\mathbb{E}_{y|x}\left[1[y=h^{(1)}(x)]-1[y=h^{(2)}(x)]+c\right]\\
 & =\Pr(y\neq h^{(1)}(x))+\mathbb{E}_{x}1[r(x)=1]\cdot\left[\eta_{h^{(1)}(x)}(x)-\eta_{h^{(2)}(x)}(x)+c\right],
\end{align*}
where we define $\eta_{y'}(x)\defEq\Pr(y=y'\mid x)$. Thus, it is
optimal to defer when
\begin{align*}
r(x)=1 & \iff\eta_{h^{(1)}(x)}(x)-\eta_{h^{(2)}(x)}(x)+c<0\\
 & \iff\eta_{h^{(2)}(x)}(x)-\eta_{h^{(1)}(x)}(x)>c.
\end{align*}

\end{proof}

\subsection{Proof of \Cref{corr:excess-risk}}
\begin{proof}[Proof of \Cref{corr:excess-risk}]
For fixed $h^{(1)},h^{(2)}$, we have already computed the Bayes-optimal
rejector $r^{*}$. Let $\alpha(x)\defEq\eta_{h^{(1)}(x)}(x)-\eta_{h^{(2)}(x)}(x)+c$.
Plugging $r^{*}$ into the risk results in Bayes-risk
\begin{align*}
R(r^{*};h^{(1)},h^{(2)}) & =\Pr(y\neq h^{(1)}(x))+\mathbb{E}_{x}1\left[r^{*}(x)=1\right]\cdot\alpha(x).
\end{align*}
 The excess risk for an arbitrary $r$ is thus
\begin{align*}
 & R(r;h^{(1)},h^{(2)})-R(r^{*};h^{(1)},h^{(2)})\\
 & =\mathbb{E}_{x}\left[1\left[r(x)=1\right]-1\left[\alpha(x)<0\right]\right]\cdot\alpha(x).
\end{align*}
\end{proof}

\subsection{Proof of \Cref{lemm:when-confidence-suffice}}
We start with \Cref{lem:same_acc} which will help prove \Cref{lemm:when-confidence-suffice}.
\begin{lemma}
\label{lem:same_acc}
Given two base classifiers $h^{(1)},h^{(2)}\colon\mathscr{X}\to[L]$,
two deferral rules $r_{1},r_{2}:\mathscr{X}\to\{0,1\}$ yield the
same accuracy for the cascade if and only if $\mathbb{E}\left( r_{1}(x)\beta(x) \right) = \mathbb{E}\left( r_{2}(x)\beta(x) \right)$,
where $\beta(x)\defEq\eta_{h^{(2)}(x)}(x)-\eta_{h^{(1)}(x)}(x)$.
\end{lemma}
\begin{proof}
For a deferral rule $r$, by definition, the accuracy is given by
\begin{align*}
A(r) & =\mathbb{E}_{(x,y)}\left[\boldsymbol{1}[h^{(1)}(x)=y] \cdot (1-r(x))+\boldsymbol{1}[h^{(2)}(x)=y] \cdot r(x)\right]\\
 & =\mathbb{P}(h^{(1)}(x)=y)+\mathbb{E}\left[r(x) \cdot \left(\boldsymbol{1}[h^{(2)}(x)=y]-\boldsymbol{1}[h^{(1)}(x)=y]\right)\right]\\
 & =\mathbb{P}(h^{(1)}(x)=y)+\mathbb{E}_{x}r(x)\mathbb{E}_{y|x}\left[\boldsymbol{1}[h^{(2)}(x)=y]-\boldsymbol{1}[h^{(1)}(x)=y]\right]\\
 & =\mathbb{P}(h^{(1)}(x)=y)+\mathbb{E}_{x}r(x) \cdot \left(\eta_{h^{(2)}(x)}(x)-\eta_{h^{(1)}(x)}(x)\right).
\end{align*}
Given two deferral rules $r_{1},r_{2}$, $A(r_{1})=A(r_{2})\iff\mathbb{E}\left( r_{1}(x)\beta(x) \right) =\mathbb{E}\left( r_{2}(x)\beta(x) \right)$
where $\beta(x)\defEq\eta_{h^{(2)}(x)}(x)-\eta_{h^{(1)}(x)}(x)$.
\end{proof}

We are ready to prove \Cref{lemm:when-confidence-suffice}.

\begin{proof}[Proof of \Cref{lemm:when-confidence-suffice}]
Recall the confidence-based deferral rule and the Bayes-optimal rule
are respectively
\begin{align*}
r_{\mathrm{conf}}(x) & =\boldsymbol{1}\left[\eta_{h^{(1)}(x)}(x)<c'\right],\\
r^{*}(x) & =\boldsymbol{1}\left[\eta_{h^{(1)}(x)}(x)-\eta_{h^{(2)}(x)}(x)<c\right].
\end{align*}
For brevity, we use $\eta_{i}(x)$ and $\eta_{h^{(i)}(x)}(x)$ interchangeably.
Define real-valued random variables $z\defEq\eta_{1}(x)$, and $z^{*}\defEq\eta_{1}(x)-\eta_{2}(x)$
where $x\sim\mathbb{P}_{x}$. Let $\rho(c')\defEq\mathbb{E}\left(\boldsymbol{1} \left[\eta_{1}(x)<c'\right] \right)=\mathbb{P}\left(\eta_{1}(x)<c'\right)$
be the deferral rate of the confidence-based deferral rule with threshold
$c'$. Similarly, define $\rho^{*}(c)\defEq \mathbb{E}(\boldsymbol{1}\left[\eta_{1}(x)-\eta_{2}(x)<c\right]).$
Let $\gamma_{\alpha},\gamma_{\alpha}^{*}$
be the $\alpha$-quantile of the distributions of $z$ and $z^{*}$,
respectively. By definition, $\rho(\gamma_{\alpha})=\rho^{*}(\gamma_{\alpha}^{*})=\alpha$.

Let $A(r,c')$ denote the accuracy of the cascade with the deferral
rule $r$ and the deferral threshold $c'$. 
Formally, the deferral curve of $r_{\mathrm{conf}}$ is the set of
deferral rate-accuracy tuples $\{(\rho(c'),A(r_{\mathrm{conf}},c'))\mid c'\in[0,1]\}$.
The same deferral curve may be generated with $\{(\alpha,A(r_{\mathrm{conf}},\gamma_{\alpha})) \mid\alpha\in[0,1]\}$.
Similarly, for the Bayes-optimal rule, the deferral curve is defined
as $\{(\alpha,A(r^{*},\gamma_{\alpha}^{*})\mid\alpha\in[0,1]\}$.

To show that the two deferral rules produce the same deferral curve,
we show that $A(r_{\mathrm{conf}},\gamma_{\alpha})=A(r^{*},\gamma_{\alpha}^{*})$
for any $\alpha\in[0,1]$. By Lemma \ref{lem:same_acc}, this is equivalent
to showing that 
\begin{align}
\mathbb{E}\left(\boldsymbol{1}[\eta_1(x) <\gamma_{\alpha}] \cdot \beta(x) \right) & =\mathbb{E}\left(\boldsymbol{1}[ \eta_1(x) - \eta_2(x) <\gamma_{\alpha}^{*}] \cdot \beta(x)\right),\label{eq:chow_oracle_equiv_cond}
\end{align}
where $\beta(x)\defEq\eta_{2}(x)-\eta_{1}(x)$.

Suppose that $\eta_{1}(x)$ and $\eta_{1}(x)-\eta_{2}(x)$ produce
the same ordering over instances $x\in\mathscr{X}$. This means that
for any $x\in\mathscr{X}$ and $\alpha\in[0,1]$, $\eta_{1}(x) <\gamma_{\alpha}\iff  \eta_{1}(x) - \eta_{2}(x) < \gamma_{\alpha}^{*}$.
Thus, (\ref{eq:chow_oracle_equiv_cond}) holds. 
\end{proof}

\subsection{Proof of Lemma \ref{lem:finite-sample}}
\label{app:finite-sample}
We provide an excess risk bound in Lemma \ref{lem:excess-risk} and generalization bound in Lemma \ref{lem:def-genbound}. The proof of Lemma \ref{lem:finite-sample} follows from combining the two results.

\textbf{Excess risk bound.} The excess risk for the learned deferral rule can be bounded as follows:
\begin{lemma}
\label{lem:excess-risk} 
The excess risk for $\hat{r}$ is bounded by:
$$
R(\hat{r};h^{(1)},h^{(2)}) - R(r^*;h^{(1)},h^{(2)})
\leq 
2\sqrt{\E_x\left[
      ( \hat{g}(x)-g^*(x) )^2 
    \right]}.
$$
\end{lemma}
\begin{proof}
We first expand the risk for deferral rule $\hat{r}$:
\begin{align*}
 R(\hat{r};h^{(1)},h^{(2)}) &= 
 \E\left[
     1[y\neq h^{(1)}(x),\hat{g}(x) \leq c]
        +
     1[y\neq h^{(2)}(x),\hat{g}(x) > c]+c \cdot 1[\hat{g}(x) > c]
    \right]\\
&= \E_x\left[
     ( 1 - \eta_{h^{(1)}(x)}(x) )\cdot 1[\hat{g}(x) \leq c]
        +
     (1 - \eta_{h^{(2)}(x)}(x) + c)\cdot 1[\hat{g}(x) > c]
    \right]\\
&= \E_x\left[
     (\eta_{h^{(1)}(x)}(x) - \eta_{h^{(2)}(x)}(x) + c)\cdot 1[\hat{g}(x) > c]
    \right]
    + \E_x\left[1 - \eta_{h^{(1)}(x)}(x)\right]\\
&= \E_x\left[
     (-g^*(x) + c)\cdot 1[\hat{g}(x) > c]
    \right]
    + \E_x\left[1 - \eta_{h^{(1)}(x)}(x)\right].
\end{align*}
Per Corollary~\ref{corr:excess-risk}, the excess risk for $\hat{r}$ can then be written as:
\begin{align*}
 \lefteqn{ R(\hat{r};h^{(1)},h^{(2)}) - R(r^*;h^{(1)},h^{(2)}) }\\
&= \E_x\left[
     (-g^*(x) + c)\cdot 1[\hat{g}(x) > c]
    \right]
    \,-\,
    \E_x\left[
     (-g^*(x) + c)\cdot 1[g^*(x) > c]
    \right]\\
&= \E_x\left[
     (-g^*(x) + c)\cdot 1[\hat{g}(x) > c]
    \right]
    \,-\,
    \E_x\left[
     (-\hat{g}(x) + c)\cdot 1[g^*(x) > c]
    \right]\\
&\hspace{1cm}
    \,+\,
    \E_x\left[
     (-\hat{g}(x) + c)\cdot 1[g^*(x) > c]
    \right]
    \,-\,
    \E_x\left[
     (-g^*(x) + c)\cdot 1[g^*(x) > c]
    \right].
\end{align*}
We next bound the second term on the right-hand side. Note that
$
(-\hat{g}(x) + c)\cdot 1[g'(x) > c]
$
is minimized when $g'(x) = \hat{g}(x)$. Therefore
$
(-\hat{g}(x) + c)\cdot 1[g^*(x) > c] 
\geq
(-\hat{g}(x) + c)\cdot 1[\hat{g}(x) > c] 
$. 
Plugging this into the above inequality gives us:
\begin{align*}
 \lefteqn{ R(\hat{r};h^{(1)},h^{(2)}) - R(r^*;h^{(1)},h^{(2)}) }\\
&\leq \E_x\left[
     (-g^*(x) + c)\cdot 1[\hat{g}(x) > c]
    \right]
    \,-\,
    \E_x\left[
     (-\hat{g}(x) + c)\cdot 1[\hat{g}(x) > c]
    \right]\\
&\hspace{1cm}
    \,+\,
    \E_x\left[
     (-\hat{g}(x) + c)\cdot 1[g^*(x) > c]
    \right]
    \,-\,
    \E_x\left[
     (-g^*(x) + c)\cdot 1[g^*(x) > c]
    \right]\\
&= \E_x\left[
     ( \hat{g}(x)-g^*(x) )\cdot 1[\hat{g}(x) > c]
    \right]
    \,+\,
    \E_x\left[
     (g^*(x)-\hat{g}(x))\cdot 1[g^*(x) > c]
    \right]
    \\
&\leq
    \E_x\left[
     | \hat{g}(x)-g^*(x) |\cdot (1[\hat{g}(x) > c]
    \,+\, 1[g^*(x) > c])
    \right]
\\
&\leq 
    2 \cdot \E_x\left[
     | \hat{g}(x)-g^*(x) |
    \right]
\\
&\leq
    2 \cdot \E_x\left[
     \sqrt{ ( \hat{g}(x)-g^*(x) )^2 }
    \right]
\\
&\leq
    2 \cdot \sqrt{\E_x\left[
      ( \hat{g}(x)-g^*(x) )^2 
    \right]},
\end{align*}
where the second-last step follows from Jensen's inequality and the fact that $\sqrt{\cdot}$ is concave.
\end{proof}

\textbf{Finite-sample bound.} Next, we bound the right-hand side of the excess risk bound in Lemma \ref{lem:excess-risk} with a finite-sample bound. 
For this, define the population squared loss for a scorer $g$ by:
\begin{align*}
    {R}_{\rm sq}(g) &= \E\left[ ( z - g( x ) )^2 \right],
\end{align*}
where $z = 1[ y = h^{(2)}( x )] - 1[y = h^{(1)}( x ) ].$ 
It is straight-forward to see that the minimiser of this loss over all measurable functions $g: \XCal\rightarrow \R$ is $\E[z|x] =  \eta_{h^{(2)}(x)}(x)-\eta_{h^{(1)}(x)}(x) = g^*(x)$. Let $\hat{g}$ be minimizer of the average squared loss 
$\frac{1}{| S_{\rm{val}} |} \sum_{( x_i, z_i ) \in S_{\rm val}} ( z_i - g( x_i ) )^2$ over hypothesis class $\mathcal{G}$, where the labels in  $S_{\rm{val}}$ are given by $z_i = 1[ y_i = h^{(2)}( x_i )] - 1[y_i = h^{(1)}( x_i ) ]$. 

\begin{lemma}
\label{lem:def-genbound}
Let $\mathcal{N}(\mathcal{G}, \epsilon)$ denote the covering number for $\mathcal{G}$ with the $\infty$-norm.  Suppose for any $g \in \mathcal{G}$, the squared loss $(z - g(x))^2 \leq B, \forall (x, z)$. Furthermore, let $\tilde{g}$ denote the minimizer of the population squared loss $\E\left[(z - g(x))^2\right]$ over $\mathcal{G}$, where $z = 1[ y = h^{(2)}( x )] - 1[y = h^{(1)}( x ) ]$. Then for any $\delta \in (0,1)$, with probability at least $1 - \delta$ over draw of $S_{\rm val}$, 
\begin{align*}
\lefteqn{\E_x\left[ (\hat{g}(x)  -  g^*(x))^2\right]}\\
&\leq\,
\underbrace{
    \E_x\left[ (\tilde{g}(x)  -  g^*(x))^2\right]}_{\text{Approximation error}}  + 
    \underbrace{4\cdot\inf_{\epsilon > 0}
    \left\{
        B\sqrt{
            \frac{ 2\cdot\log\mathcal{N}(\mathcal{G}, \epsilon) }{|S_{\rm val}|}
        }
    \right\}
    \,+\,
    \mathcal{O}\left(\sqrt{\frac{\log(1/
    \delta)}{|S_{\rm val}|}}\right)}_{\text{Estimation error}}. 
\end{align*}
\end{lemma}
\begin{proof}
We have from standard uniform convergence arguments (Shalev-Shwartz \& Ben-David, 2014), the following generalization bound for the squared loss: with probability at least $1 - \delta$ over draw of $S_{\rm val}$, for any $g \in \mathcal{G}$:
\begin{align}
\left|R_{\rm sq}(g) - \hat{R}_{\rm sq}(g) \right|
    \,\leq\, 
    2\cdot\inf_{\epsilon > 0}
    \left\{
        \epsilon + 
        B\sqrt{
            \frac{ 2\cdot\log\mathcal{N}(\mathcal{G}, \epsilon) }{|S_{\rm val}|}
        }
    \right\}
    \,+\,
    \mathcal{O}\left(\sqrt{\frac{\log(1/
    \delta)}{|S_{\rm val}|}}\right).
    \label{eq:R-sq-gen-bound}
\end{align}
Expanding the population squared loss for a scorer $g$, we have:
\begin{align*}
    R_{\rm sq}(g) 
    &=  \E_x\left[ \E_{z|x}\left[ (z - g(x))^2 \right] \right]\\
    &= \E_x\left[ \E_{z|x}\left[ z^2 \right]  - 2 \cdot \E_{z|x}\left[ z \right] \cdot g(x) + g(x)^2\right]\\
    &= \E_x\left[ \E_{z|x}\left[ z \right]  - 2 \cdot \E_{z|x}\left[ z \right] \cdot g(x) + g(x)^2\right] + \E\left[z^2 - z\right]\\
    &= \E_x\left[ (\E_{z|x}\left[ z \right]  -  g(x))^2\right] + \E\left[z^2 - z\right]\\
    &= \E_x\left[ (g^*(x)  -  g(x))^2\right] + \E\left[z^2 - z\right].
\end{align*}
Notice that the second term is independent of $g$. 
The excess squared loss for any scorer $g$ can then be written as:
\begin{align}
    R_{\rm sq}(g) - R_{\rm sq}(g^*) = \E_x\left[ (g(x)  -  g^*(x))^2\right].
    \label{eq:excess-sq-risk}
\end{align}

It then follows that:
\begin{align*}
 \E_x\left[ (\hat{g}(x)  -  g^*(x))^2\right] 
 &= R_{\rm sq}(\hat{g}) - R_{\rm sq}(g^*)\\
 &= R_{\rm sq}(\hat{g}) - R_{\rm sq}(\tilde{g}) + R_{\rm sq}(\tilde{g}) - R_{\rm sq}(g^*)\\
  &\stackrel{(a)}{=} R_{\rm sq}(\hat{g}) - R_{\rm sq}(\tilde{g}) + \E_x\left[ (\tilde{g}(x)  -  g^*(x))^2\right]\\
 &= R_{\rm sq}(\hat{g}) - \hat{R}_{\rm sq}(\tilde{g}) + \hat{R}_{\rm sq}(\tilde{g}) - R_{\rm sq}(\tilde{g}) + \E_x\left[ (\tilde{g}(x)  -  g^*(x))^2\right]\\
 & \stackrel{(b)}{\leq} R_{\rm sq}(\hat{g}) - \hat{R}_{\rm sq}(\hat{g}) + \hat{R}_{\rm sq}(\tilde{g}) - R_{\rm sq}(\tilde{g}) + \E_x\left[ (\tilde{g}(x)  -  g^*(x))^2\right]\\
 &\leq |R_{\rm sq}(\hat{g}) - \hat{R}_{\rm sq}(\hat{g})| + |\hat{R}_{\rm sq}(\tilde{g}) - R_{\rm sq}(\tilde{g})| + \E_x\left[ (\tilde{g}(x)  -  g^*(x))^2\right]\\
 &\leq 2\cdot \sup_{g \in \mathcal{G}}\, |{R}_{\rm sq}(g) - \hat{R}_{\rm sq}(g)| + \E_x\left[ (\tilde{g}(x)  -  g^*(x))^2\right],
\end{align*}
where step $(a)$ uses \eqref{eq:excess-sq-risk};
step $(b)$ uses the fact that $\hat{g}$ is the minimizer of $\hat{R}_{\rm sq}$ over $\mathcal{G}$. 
Upper bounding the right-hand side using the generalization bound in \eqref{eq:R-sq-gen-bound} completes the proof.
\end{proof}

\section{Amount of Compute Used for Experiments}
There are two types of models involved in all experiments: 
(1) base classifiers (i.e., $h^{(1)}$ and $h^{(2)}$), and 
(2) post-hoc model. 
For post-hoc model training and evaluation, we use one Nvidia V100 GPU. 
As discussed in \Cref{sec:post_hoc_model_params}, our post-hoc model is only a small
MLP model with only two hidden layers. In each experiment reported in the main text, training one post-hoc model for 20 epochs only takes a few minutes.
For training and evaluating a post-hoc model, a GPU is  not needed.

For training of base classifiers, the amount of compute varies depending on the dataset and model architecture. This is summarized in the following table. In the following table, a GPU always refers
to an Nvidia V100 GPU, and a TPU always refers to a Google Cloud TPU v3.\footnote{Google Cloud TPU v3: \url{https://cloud.google.com/tpu/docs/system-architecture-tpu-vm.}}

\begin{center}
\begin{tabular}{lllp{7cm}}
\toprule 
Dataset & Model & Devices & Approximate Training Time
\tabularnewline
\midrule 
CIFAR 100  & CIFAR ResNet 8 & $8\times$ GPUs & 12m (batch size: 1024, 256 epochs) \\
CIFAR 100  & CIFAR ResNet 14 & $8\times$ GPUs & 20m (batch size: 1024, 256 epochs)\\
CIFAR 100  & CIFAR ResNet 56 & $8\times$ GPUs & 20m (batch size: 1024, 256 epochs)\\
ImageNet  & MobileNet V2 & $8\times$  TPUs & 7h (batch size: 64, 90 epochs) \\
ImageNet  & EfficientNet B0 & $8\times$  TPUs & 4h (batch size: 1024, 90 epochs) \\

\bottomrule
\end{tabular}
\end{center}

\newpage
\section{Experimental Setup: Hyper-parameters}
\label{app:hyperparams}

\subsection{Training of Models in a Cascade}

We describe hyperparameters we used for training all base models (i.e., $h^{(1)}$ and $h^{(2)}$).
In the following table, BS denotes batch size, and schedule refers
to learning rate schedule. 

\begin{center}
\begin{tabular}{>{\raggedright}p{4cm}lcccc}
\toprule 
Dataset & Model & LR & Schedule & Epochs & BS\tabularnewline
\midrule 
\multirow{2}{2.5cm}{Standard ImageNet} & MobileNet V2 & 0.05 & anneal & 90 & 64\tabularnewline
 & EfficientNet B0 & 0.1 & cosine & 90 & 1024\tabularnewline
\midrule
ImageNet Dog (specialist) & EfficientNet B0 & 0.1 & cosine & 90 & 512\tabularnewline
\midrule
\multirow{2}{2.5cm}{CIFAR 100} & CIFAR ResNet 8 & 1.0 & anneal & 256 & 1024\tabularnewline
 & CIFAR ResNet 14 & 1.0 & anneal & 256 & 1024\tabularnewline
 & CIFAR ResNet 56 & 1.0 & anneal & 256 & 1024\tabularnewline

\bottomrule
\end{tabular}
\end{center}

We use SGD with momentum as the optimisation
algorithm for all models. For annealing schedule, the specified learning
rate (LR) is the initial rate, which is then decayed by a factor of
ten after each epoch in a specified list. For CIFAR100, these epochs
are 15, 96, 192 and 224. For ImageNet, these epochs are 5, 30, 60,
and 80. 
In the  above table, CIFAR 100 includes the standard CIFAR 100 classification problem, CIFAR 100 with label noise, and CIFAR 100 with distribution shift as discussed in \Cref{fig:three_settings_grid}.

\subsection{Training of Post-hoc Deferral Rules}
\label{sec:post_hoc_model_params}
All post-hoc models $g\colon\XCal\to\mathbb{R}$ we consider are based
on a lightweight Multi-Layer Perceptron (MLP) that takes as input a small set of features constructed from  probability outputs from model 1. 
More precisely, let $p^{(1)}(x)\in\Delta_{L}$ denote all probability outputs from model 1 for an $L$-class classification problem. 
Let $v(p^{(1)}(x)) \in \mathbb{R}^D$ be a list of features extracted from the probability outputs. 
In all experiments, the post-hoc model is $g(x)=\mathrm{MLP}(v(p^{(1)}(x)))$ with $v$ producing $D=L+11$ features.  These features are
\begin{enumerate}
    \item The entropy of $p^{(1)}(x)$.
    \item Top 10 highest probability values of $p^{(1)}(x)$.
    \item One-hot encoding of $\arg \max_{y'} p^{(1)}_{y'}$ (i.e., an $L$-dimensional binary vector).
\end{enumerate}

For the MLP, let $\mathrm{FC}_{K}$ denote a fully connected layer
with $K$ output units without an activation. Let $\mathrm{FC}_{K,f}$ denote
a fully connected layer with $K$ output units with $f$ as the activation.
In all experiments involving post-hoc rules, we use 
\begin{align*}
    g(x) = (\mathrm{FC}_{1}\circ\mathrm{FC}_{2^4,\mathrm{ReLU}}\circ\mathrm{FC}_{2^6,\mathrm{ReLU}} \circ v \circ p^{(1)})(x),
\end{align*}
where  ReLU denotes the Rectified Linear Unit.
Note that both the MLP and the set of input features to $g$ are small. We found that
a post-hoc model can easily overfit to its training set. Controlling its capacity
helps mitigate this issue.

To train $g$, we use Adam \citep{kingma2014adam}
as the optimisation algorithm with a constant learning rate of 0.0007
and batch size 128. For CIFAR 100 experiments, we use a held-out set of size 5000
as the training set. For ImageNet experiments, we use an held-out set of size 10000.
We train for 20 epochs.
Further, an L2 regularization of weights in each layer in the MLP is also added to the training objective. We set the  regularization weight to 0.001.

\newpage
\section{Calibration Analysis}
\label{app:calibration}

To further study the performance of confidence-based deferral, 
Figure \ref{fig:m0w_m1r_chow} presents \emph{calibration plots}~\citep{Niculescu-Mizil:2005,Guo:2017}.
These ideally visualise 
$\Pr( y \neq h^{(1)}( x ) \land y = h^{(2)}( x ) \mid x \in \mathscr{X}_q )$ for $q \in [ 0, 1 ]$,
where $\mathscr{X}_q \defEq \{ x \in \mathscr{X} \colon \max_{y'} p^{(1)}_{y'}( x ) = q \}$ is the set of samples where model 1's confidence equals $q$.
In practice, we discretise $q$ into $10$ buckets.

\begin{figure}[!h]
    \centering 
     \subcaptionbox{MobileNet V2 and EfficientNet B0 on ImageNet.}{\includegraphics[scale=0.4]{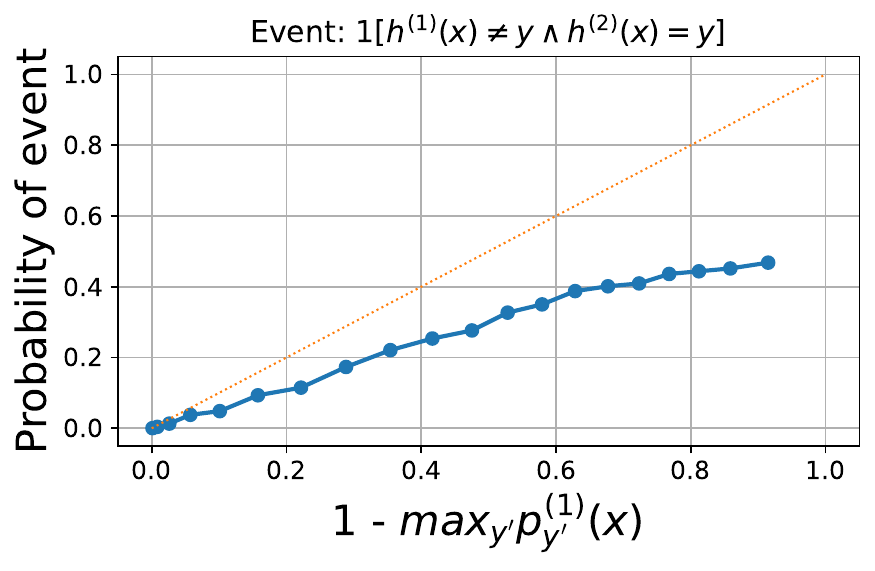}}
    \qquad
     \subcaptionbox{MobileNet V2 and \textbf{dog specialist} EfficientNet B0 on ImageNet.}{\includegraphics[scale=0.4]{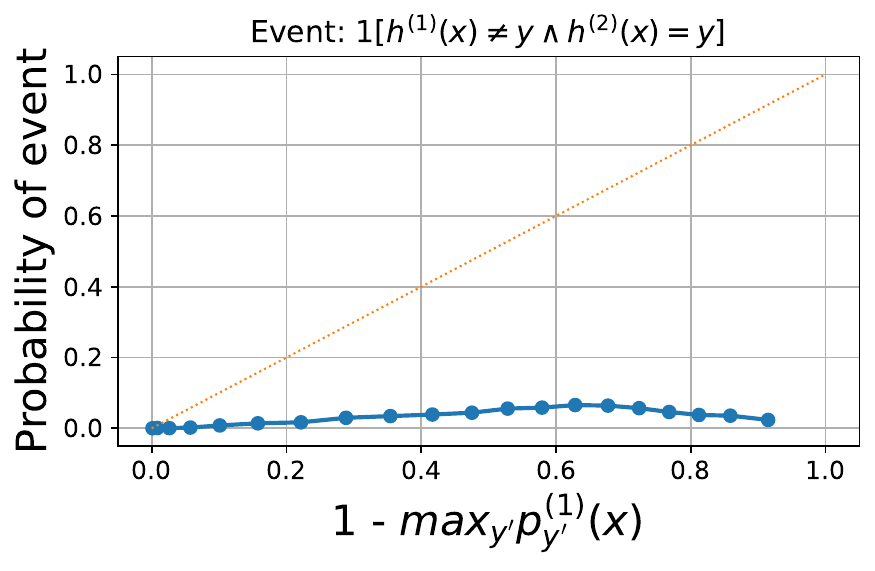}}
 
    \caption{Calibration plots visualising the empirical probability of the event 
    $h_1(x) \neq y \land h_2(x) = y$, 
    i.e., the first model predicts incorrectly and the second model predicts correctly. 
    In settings where confidence-based deferral performs poorly, the first model's confidence is a poor predictor of whether the second model makes a correct prediction: 
    the likelihood of this event tends to be systematically \emph{over-estimated}.}
    \label{fig:m0w_m1r_chow}
\end{figure}

These plots reveal that, as expected,
the confidence of $h^{(1)}$ may be a poor predictor of when $h^{(1)}$ is wrong but $h^{(2)}$ is right:
indeed, it systematically \emph{over-estimates} this probability, and thus may result in erroneously forwarding samples to the second model.
This suggests using post-hoc estimates of when $h^{(1)}$ is wrong, but $h^{(2)}$ is right.

We remark here that in a different context,~\citet{narayan2022} showed that certain cascades of pre-trained language models may be amenable to confidence-based deferral:
the confidence of the small model can predict samples where the small model is wrong, but the large model is correct.
In this work, we focus on image classification settings. 
Generalising our analysis to natural langauge processing models will be an interesting
topic for future study.

\section{Additional Experimental Results}
\label{app:more_results}

In this section, we present more experimental results to support our analysis in \Cref{sec:conf_vs_oracle}.

\subsection{Confidence-Based Deferral Underperforms Under Label Noise}
As illustrated in \S\ref{sec:chow_suffices}, confidence-based deferral
can be sub-optimal when there is label noise. 
The label noise experiment in \S\ref{sec:chow_suffices} is based on  synthetic label noise injected to clean data to have a controlled experiment.
The aim of this section is to confirm that we have the same conclusion with a real problem that has label noise.
Here, we consider the Mini-ImageNet dataset (with noise rate 60\%) from \cite{jiang2020}. 
The two base models in the cascade are two  ResNet 10 models with widths 16 (small model) and 64 (large model). Figure~\ref{fig:mini_imagenet_label_noise} shows the results.
As in the main text, \texttt{Diff-01} significantly outperforms Confidence based deferral in this setting.

\begin{figure}[h]
 \centering
\includegraphics[width=0.5\textwidth]{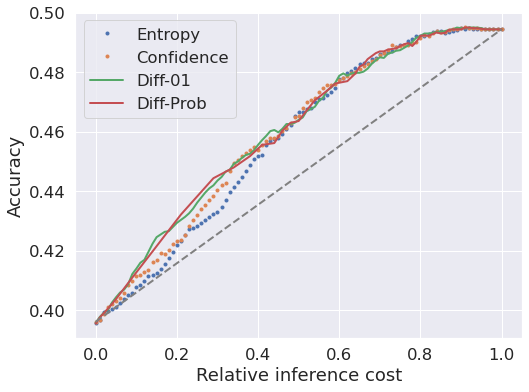}
 \caption{Test accuracy vs deferral rate of the post-hoc approaches (\texttt{Diff-01}, \texttt{Diff-Prob}), confidence thresholding (Confidence), and entropy thresholding (Entropy).
 We considered the Mini-ImageNet dataset (with noise rate 60\%) from \cite{jiang2020}. The two base models are ResNet 10 models with widths 16 (small model) and 64 (large model). Consistent with our analysis, confidence-based deferral underperforms in the presence of label noise.
 }
 \label{fig:mini_imagenet_label_noise}
\end{figure}

\subsection{When Model 2 is Highly Accurate}

\begin{figure*}[h]
    \centering
    \resizebox{0.59\linewidth}{!}{
     \includegraphics[scale=0.39]{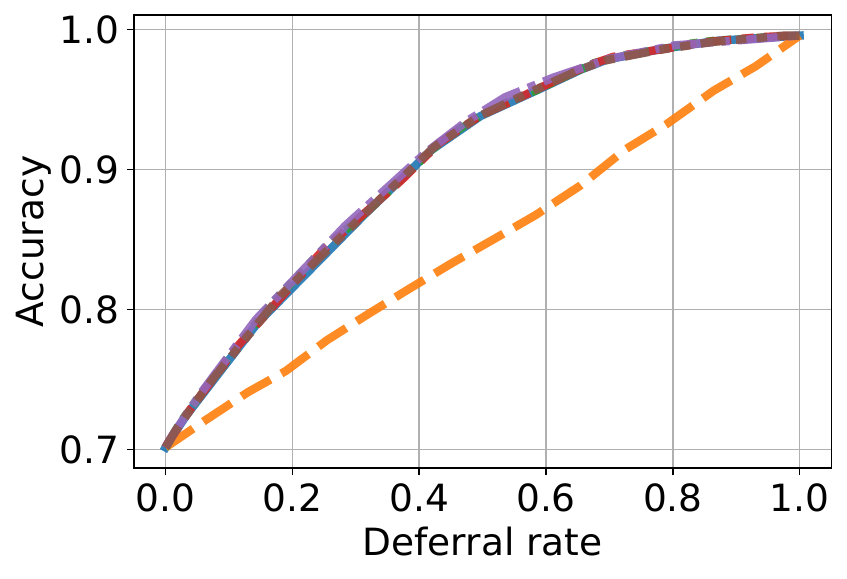}
     \includegraphics[scale=0.60,trim={0 -15mm 0 0},clip]{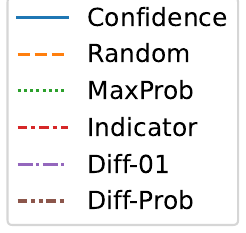}
    }
    \caption{Test accuracy vs deferral rate of post-hoc deferral rules (in \Cref{tbl:post_hoc_estimates}) when the second model is highly accurate. We observe that confidence-based deferral
    is competitive in this setting.
    }
    \label{fig:mnist_posthoc}
\end{figure*}

To illustrate that confidence-based deferral is optimal when the second model has
a constant error probability, we train an MLP model as $h^{(1)}$ and a CIFAR ResNet 56 as $h^{(2)}$ on the MNIST dataset \citep{lecun1998gradient}, a well known 10-class classification problem.
The MLP has one hidden layer composed of two hidden units with a ReLU activation, 
and a fully connected layer with no activation function to produce 10 logit scores (for 10 classes).
\Cref{fig:mnist_posthoc} presents our results.

\subsection{Oracle Curves}

\begin{figure*}[h]
    \centering
    \resizebox{0.69\linewidth}{!}{
    \label{fig:plug_oracle_specialist}
    \includegraphics[scale=0.42]{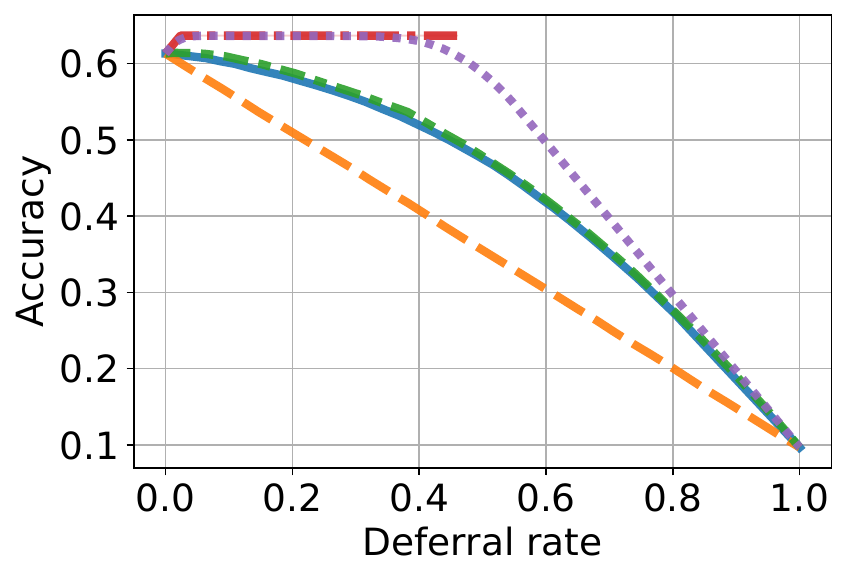}
    \includegraphics[scale=0.55,trim={0 -15mm 0 0},clip]{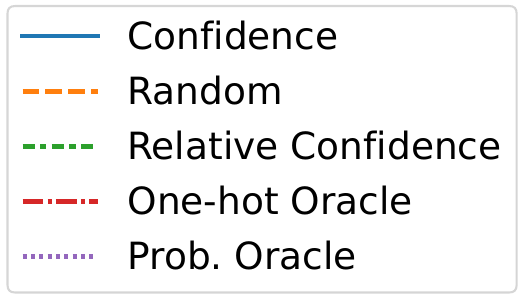}
    }
    \caption{Test accuracy vs deferral rate of different plug-in estimates
    for the oracle deferral rule in \eqref{eq:optimal_rule} (see \Cref{sec:plug_in_rule}). 
    The first model is trained on all ImageNet classes and the second model specialises
    on dog classes (i.e., and hence has non-uniform errors).
    In this case, accounting for the label probability of $h^{(2)}$ (as done by the oracle approaches) can theoretically improve upon  confidence-based deferral.  
    }
    \label{fig:plug_oracle_rule}
    
\end{figure*}

In this section, we illustrate  different plug-in estimates
for the oracle deferral rule in \eqref{eq:optimal_rule} (see \Cref{sec:plug_in_rule}).
We consider two base models: MobileNet V2 \citep{sandler2018mobilenetv2} and EfficientNet B0 \citep{tan2019efficientnet} as $h^{(1)}$ and $h^{(2)}$, which are trained independently on ImageNet. The second model  is trained on a subset of ImageNet dataset containing only ``dog'' synset (i.e., a dog specialist model). 
\Cref{fig:plug_oracle_rule} illustrates the performance of confidence-based deferral
along with  different plug-in estimates for the oracle deferral rule.

We see that in this specialist setting, confidence-based deferral is sub-optimal,
with a considerable gap to the oracle curves;
this is  expected, since here $\Pr( y = h^{(2)}( x ) )$ is highly variable for different $x$.
In particular, $\Pr( y = h^{(2)}( x ) ) \sim 1$ for dog images, but $\Pr( y = h^{(2)}( x ) ) \sim \frac{1}{L - \text{\#non-dog classes}}$ for other images. There are 119 classes in the dog synset.
Note that all Oracle curves are theoretical constructs that rely on the true label 
$y$. 
They serve as an upper bound on the performance what we could achieve when we
learn post-hoc deferral rules to imitate them.

\subsection{Confidence of a Dog-Specialist Model}
In this section, we report confidence of the two base models we use in the ImageNet dog-specialist setting in \Cref{fig:plug_oracle_rule}.
Recall that $p^{(1)}$ is MobileNet V2 \citep{sandler2018mobilenetv2} trained on the full ImageNet dataset (1000 classes), 
and $p^{(2)}$ is EfficientNet B0 \citep{tan2019efficientnet} trained only on the 
dog synset from ImageNet (119 classes).
We show empirical distributions of $\max_{y'} p^{(1)}_{y'}(x), \max_{y'} p^{(2)}_{y'}(x), \max_{y''} p^{(2)}_{y''}(x)-\max_{y'} p^{(1)}_{y'}(x), p^{(2)}_{y}(x)-p^{(1)}_{y}(x)$,
and $p^{(2)}_{y}(x)-p^{(1)}_{y}(x)$ in \Cref{fig:conf_hist}, grouped by the category
of each input image (Dog or Non-Dog).
These statistics are computed on the test set of ImageNet dataset, containing 50000 images.

\newcommand{\wid}{0.32\textwidth}
\label{sec:model_confidence}
\begin{figure*}[h]
    \centering
    \includegraphics[width=0.31\textwidth]{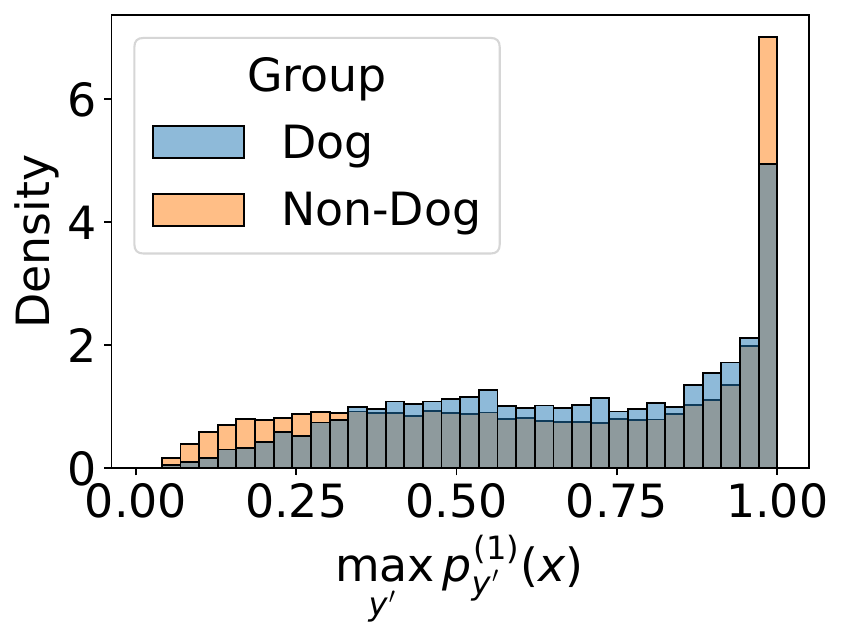}
    \includegraphics[width=\wid]{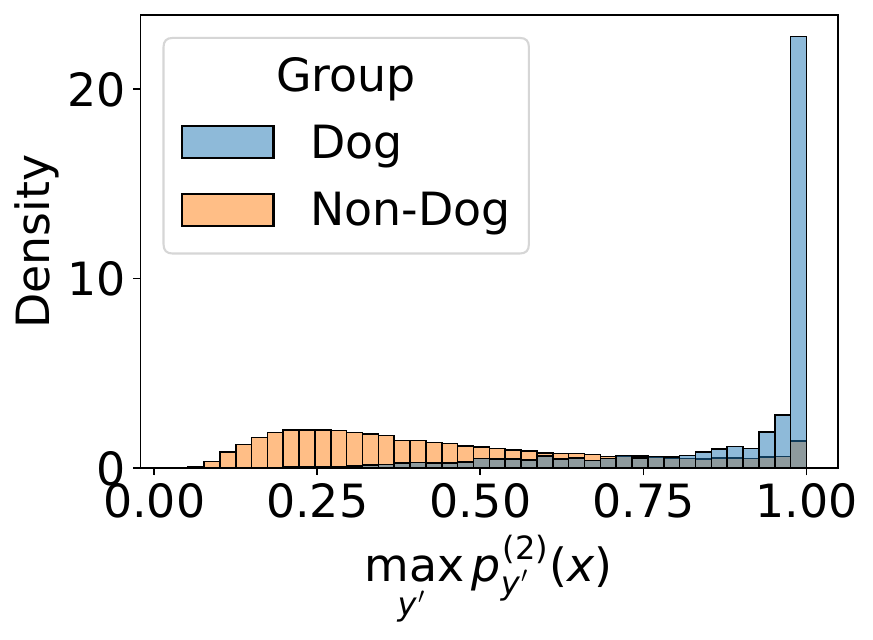}
    \includegraphics[width=\wid]{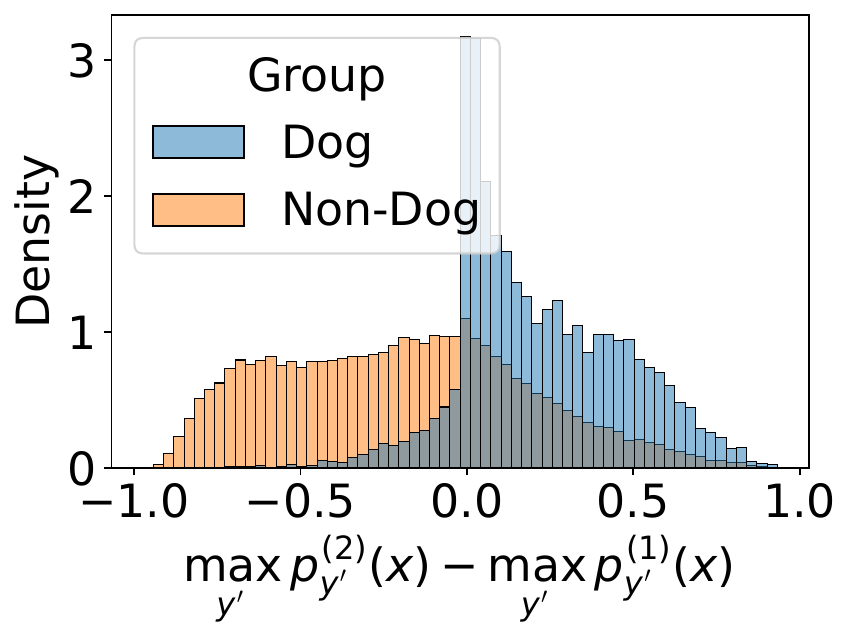} \\
    \includegraphics[width=\wid]{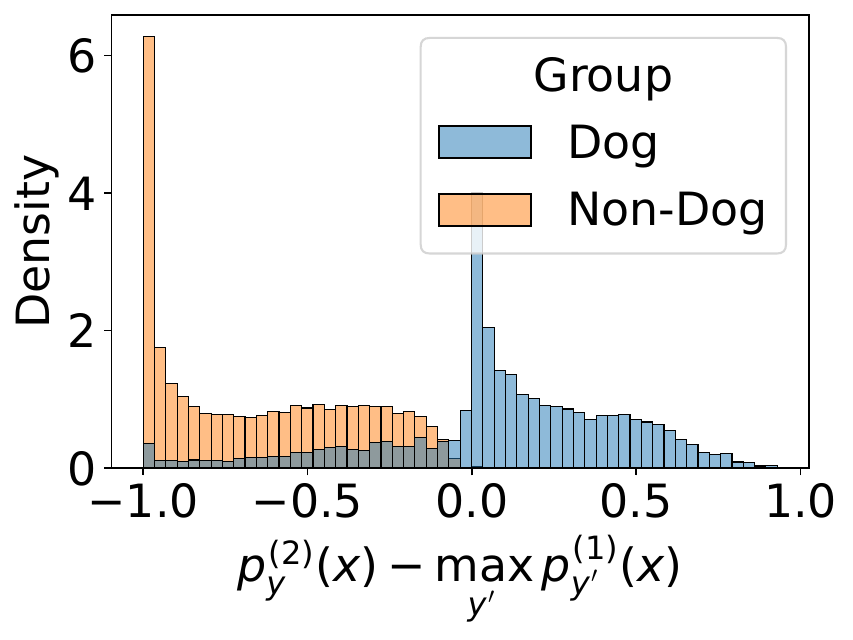}
    \includegraphics[width=\wid]{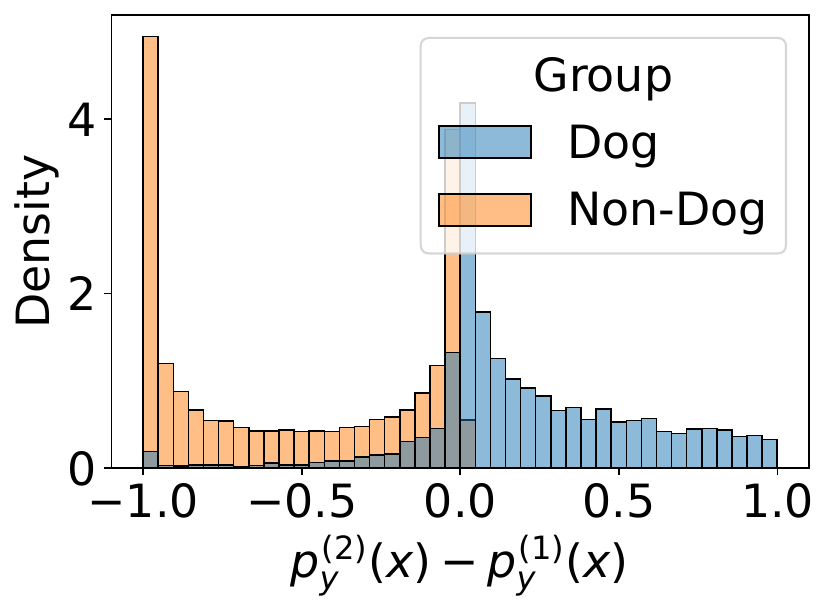}
    \caption{Confidence on ImageNet test set of $p^{(1)}$ (MobileNet V2 trained on the full ImageNet training set), and $p^{(2)}$ (EfficientNet B0 trained only on the dog synset).
    }
    \label{fig:conf_hist}
\end{figure*}

We observe from the distribution of $\max_{y'} p^{(2)}_{y'}(x)$ that the specialist EfficientNet B0
is confident on dog images. To the specialist, non-dog images are out of distribution and so 
$\max_{y'} p^{(2)}_{y'}(x)$ is not a reliable estimate of the confidence. That is, $p^{(2)}$ can be highly confident on non-dog images. As a result, a deferral rule based on Relative Confidence ($ \max_{y''} p^{(2)}_{y''}(x)-\max_{y'} p^{(1)}_{y'}(x)$) may erroneously route non-dog images
to the second model.
Thus, for training a post-hoc model, it is important to consider model 2's probability of the true label 
(i.e., $p_y^{(2)}(x)$), which is low when $(x,y)$ is a labeled image outside the dog synset.

\begin{figure*}[h]
    \centering
    \includegraphics[width=\wid]{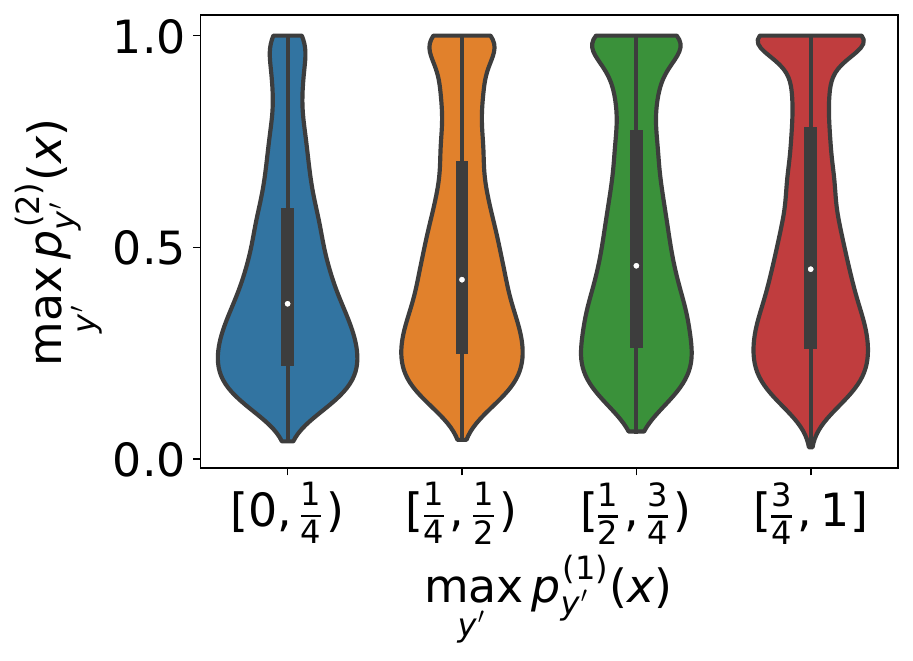}
    \includegraphics[width=\wid]{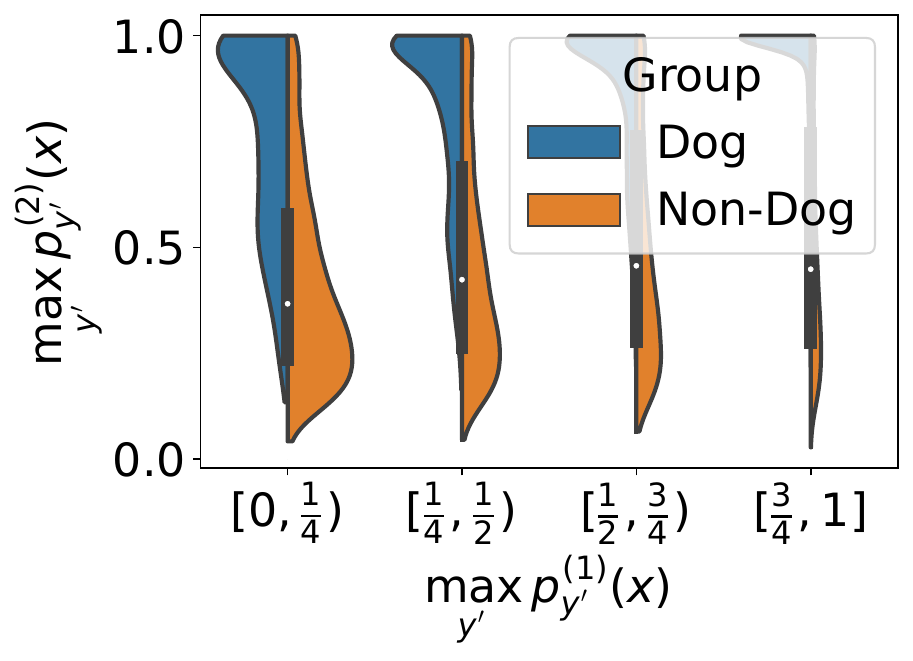}
    \includegraphics[width=\wid]{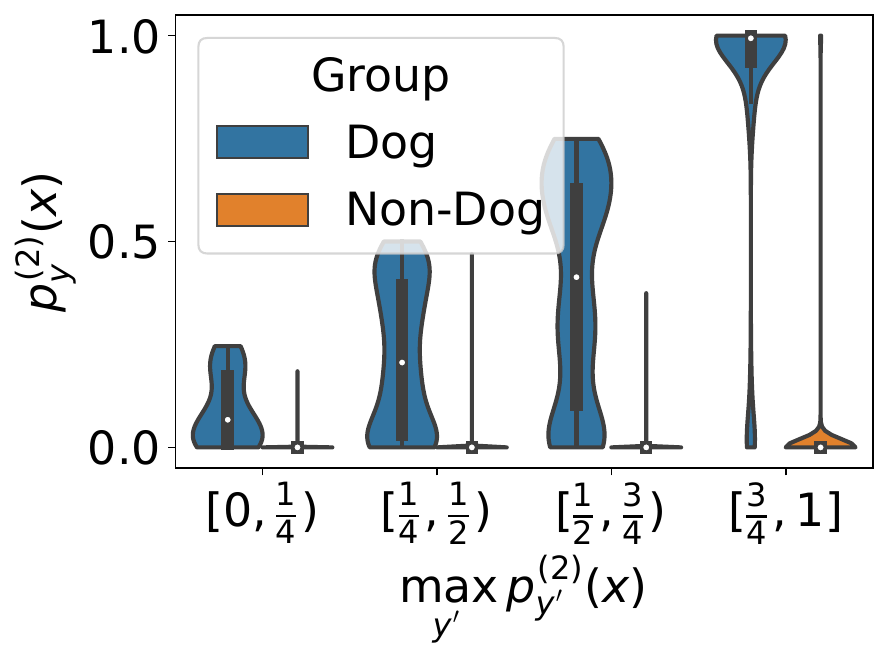}
    \caption{Confidence of MobileNet V2 (trained on the full ImageNet training set) versus confidence of  EfficientNet B0 (trained only on the dog synset).
    }
    \label{fig:conf_scatter}
\end{figure*}
We further show a scatter plot of confidence of both models in \Cref{fig:conf_scatter}.
We observe that there is no clear relationship between $\max_{y'} p^{(1)}_{y'}(x)$ and $\max_{y'} p^{(2)}_{y'}(x)$.

\section{Extension to Multi-Model Cascades}
\label{sec:multi_model_cascades}

One may readily employ post-hoc deferral rules to train multi-model cascades.
Given $K$ classifiers $ h^{(1)}, \ldots, h^{(K)} $,
a simple recipe is to train $K - 1$ deferral rules, with the $k$th rule trained to predict whether or not $h^{(k)}$ should defer to $h^{(k+1)}$.
Each of these individual rules may be trained with any of the objectives detailed in Table~\ref{tbl:post_hoc_estimates}.
At inference time, one may invoke these rules sequentially to determine a suitable termination point.

In Figure~\ref{fig:three_models_label_noise},
we present results for a 3-model cascade in the label noise setting.
As with the 2-model case, post-hoc deferral improves significantly over confidence-based deferral,
owing to the latter wastefully deferring on noisy samples that no model can correctly predict.
Here, the relative inference cost is computed as the relative cost compared to always querying the large model.

\begin{figure*}[h]
    \centering
    \includegraphics[width=0.55\textwidth]{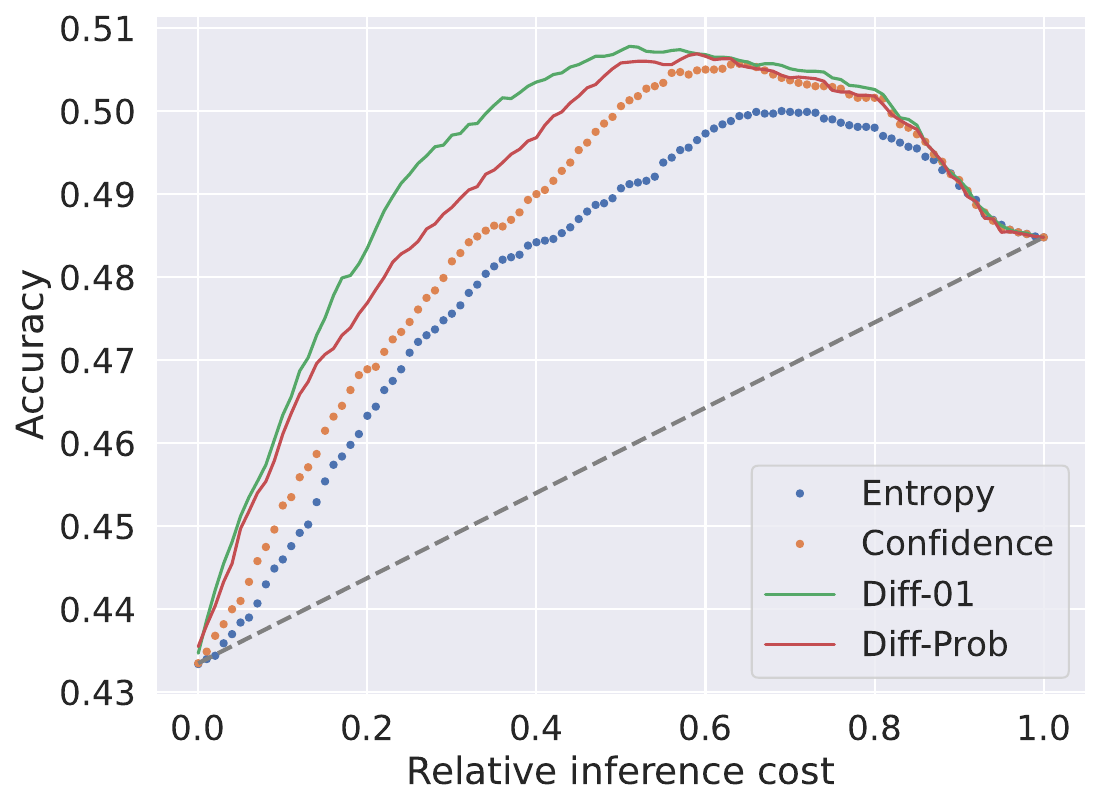}
    \caption{Test accuracy versus deferral rate for a 3-model cascade of { ResNet-8, ResNet-14, ResNet-32 }.
    Results are on CIFAR-100 with targetted noise on the first 25 classes.
    As with the 2-model cascade, post-hoc deferral improves significantly over confidence-based deferral.
    Here, the relative inference cost is computed as the relative cost compared to always querying the large model.    
    }
    \label{fig:three_models_label_noise}
    
\end{figure*}

The above can be seen as a particular implementation of the following generalisation of Proposition~\ref{prop:oracle_k2}.
Suppose we have $K$ classifiers $ h^{(1)}, \ldots, h^{(K)} $,
with inference costs $c^{(1)}, \ldots, c^{(K)} \in [0, 1]$.
We assume without loss of generality that $0 = c^{(1)} \leq c^{(2)} \leq \ldots \leq c^{(K)}$,
i.e., 
the costs reflect the \emph{excess} cost over querying the first model.

Now consider learning
a \emph{selector} $s \colon \XCal \to [ K ]$ that determines which of the $K$ classifiers is used to make a prediction
for $x \in \XCal$.
Our goal in picking such a selector is to minimise the standard misclassification accuracy under the chosen classifier, {plus} an appropriate inference cost penalty:
\begin{equation}
    \label{eqn:defer}
    R( s; h^{(1)}, \ldots, h^{(K)} )
    \defEq \Pr( y \neq h^{( s( x ) )}( x ) ) 
    + \E{}{[ c^{( s( x ) )} ]} .
\end{equation}

{We have the following, which is easily seen to generalise Proposition~\ref{prop:oracle_k2}.}

\begin{lemma}
\label{lemma:opt-selector}
The optimal selector for~\eqref{eqn:defer} is given by
\begin{equation}
    \label{eqn:posthoc}
    s^*( x ) = \underset{k \in [ K ]}{\argmin} \, \mathbb{P}{ ( y \neq h^{( k )}( x ) ) } + c^{( k )}.
\end{equation}
\end{lemma}

\begin{proof}[Proof of Lemma~\ref{lemma:opt-selector}]
Observe that
\begin{align*}
    R( s; {h}^{(1)}, \ldots, h^{(K)} ) &= \Pr( y \neq h^{( s( x ) )}( x ) ) + \lambda \cdot \E{}{ c^{( s( x ) )} } \\
    &= \E_{x}{ \E_{y \mid x}{ 1( y \neq h^{( s( x ) )}( x ) ) } + \lambda \cdot c^{( s( x ) )} }.
\end{align*}
Now suppose we minimise $s$ without any capacity restriction.
We may perform this minimisation pointwise:
for any fixed $x \in \XCal$,
the optimal prediction is thus
\begin{align*}
    s^*( x ) &= \underset{k \in [ K ]}{\argmin} \, \E_{y \mid x}{ 1( y \neq h^{( k )}( x ) ) } + \lambda \cdot c^{( k )}.
\end{align*}
\end{proof}
When $K = 2$, 
we exactly arrive at Proposition~\ref{prop:oracle_k2}.
For $K > 2$, the optimal selection is the classifier with the best balance between misclassification error and inference cost.

Lemma~\ref{lemma:opt-selector} may be seen as a restatement of~\citet[Theorem 1]{Trapeznikov:2013},
which established the Bayes-optimal classifiers for a sequential classification setting,
where each classifier is allowed to invoke a ``reject'' option to defer predictions to the next element in the sequence.
This equivalently packs together a standard classifier $h^{(k)}$ and deferral rule $r^{(k)}$ into a new classifier $\bar{h}^{(k)}$.

\section{Limitations}
In this work, we have identified problem settings where confidence-based
deferral can be sub-optimal. These problem settings are (1) specialist models, (2) label noise, and (3) distribution shift (see \Cref{fig:three_settings_grid} for experimental results in these settings).
These problem settings are not exhaustive. Identifying other conditions under
which confidence-based deferral performs poorly is an interesting direction for
future work. Another interesting topic worth investigating is finite-sample behaviors of cascades, both with confidence-based deferral and with post-hoc rules.

\end{document}